\documentclass[twoside]{article}
\usepackage[preprint]{aistats2026}

\usepackage[utf8]{inputenc} 
\usepackage[T1]{fontenc}    
\usepackage{hyperref}      
\usepackage{url}      
\usepackage{booktabs}     
\usepackage{amsfonts}       
\usepackage{nicefrac}       
\usepackage{microtype}      
\usepackage{xcolor}       
\usepackage{acronym,amsmath,graphicx,subfigure,algorithm,algorithmic,amssymb,mathtools,amsthm,wrapfig}
\usepackage{mdframed}
\usepackage{subcaption}

\usepackage[american]{babel}

\usepackage{natbib} 
\usepackage{mathtools} 
\usepackage{booktabs} 
\usepackage{tikz}

\newtheorem{theorem}{Theorem}[section]
\newtheorem{proposition}[theorem]{Proposition}
\newtheorem{lemma}[theorem]{Lemma}
\newtheorem{corollary}[theorem]{Corollary}
\newtheorem{definition}[theorem]{Definition}
\newtheorem{assumption}[theorem]{Assumption}
\newtheorem{remark}[theorem]{Remark}
\newtheorem{condition}[theorem]{Condition}

\acrodef{cnn}[CNN]{Convolutional Neural Network}
\acrodef{dp}[DP]{Differential Privacy}
\acrodef{dpl}[DPL]{differentially private learning}
\acrodef{dpsgd}[DP-SGD]{Differentially Private stochastic gradient descent}
\acrodef{fnr}[FNR]{Feature-to-Noise Ratio}
\acrodef{fl}[FL]{Federated Learning}
\acrodef{iid}[IID]{Independent and Identically Distributed}

\acrodef{non-iid}[non-IID]{Not Independent and Identically Distributed}
\acrodef{ntk}[NTK]{neural tangent kernel}
\acrodef{sgd}[SGD]{Stochastic Gradient Descent}

\begin{document}
\twocolumn[

\aistatstitle{Differential Privacy in Two-Layer Networks: How DP-SGD Harms Fairness and Robustness}

\aistatsauthor{ Ruichen Xu \And Kexin Chen}

\aistatsaddress{ rcxu642@gmail.com \And kxchen9819@gmail.com} ]

\begin{abstract}
Differentially private learning is essential for training models on sensitive data, but empirical studies consistently show that it can degrade performance, introduce fairness issues like disparate impact, and reduce adversarial robustness. The theoretical underpinnings of these phenomena in modern, non-convex neural networks remain largely unexplored. This paper introduces a unified feature-centric framework to analyze the feature learning dynamics of differentially private stochastic gradient descent (DP-SGD) in two-layer ReLU convolutional neural networks. Our analysis establishes test loss bounds governed by a crucial metric: the feature-to-noise ratio (FNR). We demonstrate that the noise required for privacy leads to suboptimal feature learning, and specifically show that: 1) imbalanced FNRs across classes and subpopulations cause disparate impact; 2) even in the same class, noise has a greater negative impact on semantically long-tailed data; and 3) noise injection exacerbates vulnerability to adversarial attacks. Furthermore, our analysis reveals that the popular paradigm of public pre-training and private fine-tuning does not guarantee improvement, particularly under significant feature distribution shifts between datasets. Experiments on synthetic and real-world data corroborate our theoretical findings.
\end{abstract}

\section{Introduction}
Modern deep learning models have demonstrated remarkable efficacy across diverse applications, including image classification \citep{he2022masked} and natural language processing \citep{vaswani2017attention}.
However, the efficacy of many deep learning applications relies heavily on datasets containing sensitive private information.
To address the inherent privacy risks, differentially private learning has emerged to train models while ensuring rigorous privacy guarantees \citep{abadi2016deep}.
A standard algorithm in this domain is differentially private stochastic gradient descent (\ac{dpsgd}) \citep{abadi2016deep}, which preserves privacy by injecting noise into network parameter updates during optimization.

Despite its strong privacy guarantees, \ac{dpsgd} introduces several notable side effects: 
\textit{1) Bad learned features} \citep{tramer2020differentially}: \ac{dpsgd} trained models may learn suboptimal features, potentially inferior to handcrafted alternatives, leading to reduced performance;
\textit{2) Disparate impact} \citep{bagdasaryan2019differential,sanyal2022unfair}: \ac{dpsgd} trained models exhibit different accuracy across different classes and subpopulation groups;
and \textit{3) Worse adversarial robustness} \citep{tursynbek2020robustness}: These models may be more vulnerable to adversarial perturbations than non-private models. 

Recent work such as \citep{de2022unlocking} demonstrated that public pretraining can significantly improve \ac{dpsgd}-trained model accuracy by 30\% compared to training from scratch, suggesting a potential avenue for mitigating these side effects.

While several pioneering studies have investigated the side effects of \ac{dpsgd} \citep{esipova2022disparate, bagdasaryan2019differential}, their analyses often require structural assumptions—such as convexity or strict smoothness. While these works provide valuable insights into specific aspects of the observed side effects, extending these classical frameworks to fully capture the non-linear dynamics of modern neural network architectures remains an open challenge. In this paper, we seek to complement these foundational efforts by identifying the root cause of these side effects within a unified framework. We argue that a feature-level analysis of the learning process offers a natural path toward such unification, concurrently illuminating the mitigating role of public pretraining \citep{tramer2022position}.
Consequently, our work is motivated by the following problem:

\begin{center}% \textit{How do ReLU neural networks learn features by \ac{dpsgd}}?
	\emph{How to theoretically explain the aforementioned phenomena in \ac{dpsgd} trained ReLU neural networks within a unified framework?}
\end{center}

This paper advances the state of the art in the following ways:

\textbf{A feature learning framework for \ac{dpsgd} trained two-layer ReLU \acp{cnn}.}
We introduce a unified framework to elucidate the previously discussed side effects observed in \ac{dpsgd}-trained two-layer ReLU \acp{cnn} by exploring the feature learning process.
Given the technical challenges posed by non-convex and non-smooth ReLU \ac{cnn} and the random \ac{dp} noise, we develop a new proof technique to derive bounds for both standard and adversarial test loss.
The high-level idea is to approximate the non-linear loss function with a piecewise linear function. 
We theoretically prove that the upper and lower bounds on test loss depend on data feature size and \ac{dpsgd} noise.

\textbf{Theoretical explanations for \ac{dpsgd} induced phenomena.}
Our framework provides formal theoretical explanations for the causes of side effects and for the effectiveness of public pretraining in \ac{dpsgd} trained neural networks, using a metric, the \emph{feature-to-noise ratio (FNR)}.
In more detail,
\textit{1)} \textit{Disparate impact} in \ac{dpsgd} trained models results from an imbalanced FNRs.
Even within the same class, semantically long-tailed data with weak features is more prone to misclassification.
\textit{2)} The \ac{dpsgd} trained models exhibit \textit{worse adversarial robustness} because they learn non-robust, class-irrelevant features from the random \ac{dp} noise, which introduces an additional large error proportional to $\sqrt{Td}\sigma_n$ after $T$ iterations' training with noise standard deviation $\sigma_n$.
\textit{3)} The performance of private finetuning on \textit{publicly pretrained networks} decreases as the feature difference between pretraining and finetuning datasets increases.
This implies that public pretraining is not a panacea for mitigating the side effects of \ac{dpsgd}. 

Building upon our prior analysis, we explore strategies to improve the FNR. 
To this end, we introduce a stage-wise network freezing technique that effectively mitigates this issue and improves model performance.

\subsection{Related work}
In this section, we review related theoretical works on the side effects of differentially private learning.
Interested readers can refer to Appendix~\ref{appendix: related work} for a detailed discussion.

\textbf{Analysis of differentially private learning side effects.} 
Several studies have sought to provide theoretical explanations for the side effects associated with differentially private learning.
For example, \cite{tran2021differentially} employed Taylor expansion to investigate the disparate impact by analyzing the local loss landscape of twice differentiable loss functions during optimization.
\cite{sanyal2022unfair} studied unfairness in long-tailed data distributions in an asymptotic setting where the number of training samples tends to infinity.
\cite{zhang2022differentially} explored adversarial robustness in private linear classifiers.
\citet{wu2023augment} reveal that \ac{dpsgd} compromises the smoothness of the loss surface, necessitating specialized randomized smoothing techniques to restore certified adversarial robustness.
These analyses are often predicated on assumptions that do not readily extend to the non-convex and non-smooth nature of modern neural networks. 
Although recent theoretical analyses suggest that DP noise can improve generalization for two-layer networks in the high-dimensional regime \citep{shi2026towards}, our work highlights a complementary perspective: under moderate dimensions, DP noise can alter feature learning in ways that reduce fairness and robustness.

In contrast, this paper presents a unified framework that elucidates all the aforementioned phenomena in \ac{dpsgd}-trained two-layer ReLU \acp{cnn} by characterizing the underlying feature learning process.
\subsection{Notation}
We use lowercase letters, lowercase boldface letters, and uppercase boldface letters to denote scalars, vectors, and matrices, respectively. 
We use $[m]$ to denote the set $\{1,\cdots,m\}$.
Given two sequences $\{x_n\}$ and $\{y_n\}$, we denote $x_n = \mathcal{O}(y_n)$ if $|x_n|\le C_1|y_n|$ for some positive constant $C_1$ and $x_n = \Omega(y_n)$ if $|x_n|\ge C_2|y_n|$ for some positive constant $C_2$.
We use $x_n = \Theta(y_n)$ if both $x_n = \mathcal{O}(y_n)$ and $x_n = \Omega(y_n)$ hold.
We use $\tilde{\mathcal{O}}(\cdot), \tilde{\Theta}(\cdot), \tilde{\Omega}(\cdot)$ to omit logarithmic factors in these notations.
Given a set $\mathcal{T}$, we use $|\mathcal{T}|$ to denote its cardinality.
For a vector $\mathbf{x}\in \mathbb{R}^d$, we denote its $\ell_p (p\ge 1)$ norm as $\left\|\mathbf{x}\right\|_p = \left(\sum_{i=1}^{d}|x_i|^p\right)^{\frac{1}{p}}$.
The notation $(\mathbf{x},y)\sim\mathcal{D}$ indicates that the data sample $(\mathbf{x},y)$ is generated from a distribution $\mathcal{D}$.

\section{Model}\label{sec: model}
We analyze a one-hidden-layer \ac{cnn} trained on a structured data distribution, consistent with several prior studies  \citep{allen2020towards,jelassi2022towards,jelassi2022vision,li2023clean,zou2023benefits,huang2025does}.\footnote{For detailed experimental justifications of the data distribution, see \cite{allen2020towards}.}

\textbf{Data distribution.}
We consider a binary classification problem where each labeled sample $(\mathbf{x}, y)$ consists of a label $y \in \{1, 2\}$ and an input data vector $\mathbf{x} = (\mathbf{x}^{(1)}, \mathbf{x}^{(2)}) \in \mathbb{R}^{d \times 2}$, composed of two patches.
Each class $y$ is associated with two feature groups: majority (common) features $\mathbf{u}_{y, \text{maj}}$ and minority (rare) features $\mathbf{u}_{y, \text{min}}$. 
The data distribution $\mathcal{D}$ generates samples as follows:

\begin{mdframed}
    The label $y \in \{1, 2\}$ is sampled with a probability $p_c > 0$ for $y = 1$ and $1 - p_c$ for $y = 2$.

    Each patch $\mathbf{x}^{(1)}, \mathbf{x}^{(2)} \in \mathbb{R}^d$ is either a feature or noise patch:
    \begin{itemize}
    	\item Feature patch: One data patch ($\mathbf{x}^{(1)}$ or $\mathbf{x}^{(2)}$) is randomly selected as the feature patch. With a probability $p_f > 0.5$, it contains the majority features $\mathbf{u}_{y, \text{maj}}$; otherwise, it contains the minority features $\mathbf{u}_{y,\text{min}}$.
    	
    	\item Noise patch: The remaining patch $\boldsymbol{\xi}$ is generated from a Gaussian distribution $\mathcal{N}(0, \sigma_p^2 \mathbf{H})$, where $\mathbf{H} = \mathbf{I} - \sum_{i=1}^{2}\sum_{j\in\{\text{maj}, \text{min}\}}\mathbf{u}_{i,j}\mathbf{u}_{i,j}^\top \cdot\left\|\mathbf{u}_{i,j}\right\|_2^{-2}$.
    \end{itemize}
\end{mdframed}

Without loss of generality, we assume that the feature vectors are orthogonal, i.e., $\langle \mathbf{u}_{i,j}, \mathbf{u}_{i',j'} \rangle = 0$ for all $i, i' \in \{1,2\}$ and $ j, j'\in \{\text{maj}, \text{min}\}$ where $ (i,j) \ne (i',j')$.
Additionally, we assume that the majority and minority features satisfy $p_f\left\|\mathbf{u}_{i,\text{maj}}\right\|_2 > (1-p_f)\left\|\mathbf{u}_{i,\text{min}}\right\|_2$ for all $i\in\{1,2\}$.
The distributions $ \mathcal{D}_{i,j}$, for $i \in \{1,2\}$ and $ j\in \{\text{maj}, \text{min}\}$ is defined by the probability density:
\begin{align}\nonumber
	&\mathbb{P}_{\mathcal{D}_{i,j}}\!\left[(\mathbf{x},y)\right] \!=\! \mathbb
	{P}_{\mathcal{D}}\!\left[(\mathbf{x},y)| y=i, \text{Feature patch of } \mathbf{x} \!=\! \mathbf{u}_{i,j}\right].
\end{align}

\textbf{Learner model.}
We consider a two-layer \ac{cnn} with ReLU activation as the learner model.
The first layer consists of $m$ neurons (filters) for class $1$ and $m$ neurons for class $2$.
Each neuron processes the two input data patches separately.
The parameters of the second layer are fixed at ${1}/{m}$.
Given an input vector $\mathbf{x} = (\mathbf{x}^{(1)}, \mathbf{x}^{(2)})$, the model outputs a vector $[F_1, F_2]$, where the $k^{th}$ element is:
\begin{align}
	F_{k}\left(\mathbf{W},\mathbf{x}\right) = \frac{1}{m}\sum_{r=1}^{m}\sum_{j=1}^{2}\sigma\left(\left\langle \mathbf{w}_{k,r},\mathbf{x}^{(j)} \right\rangle\right),
\end{align}
where $\sigma(\cdot) = \max\{\cdot,0\}$ denotes the ReLU activation function, $\mathbf{W}$ represents the collection of all model weights, and $\mathbf{w}_{k,r}$ is the weight of the $r^{th}$ neuron associated with the $k^{th}$ output $F_k(\mathbf{W},\mathbf{x})$.

\begin{remark}
    The formulation in Eq. (1) explicitly captures the defining architectural constraints of a CNN: \textit{locality} and \textit{weight sharing}. Locality is enforced because kernels operate exclusively on individual local patches $\mathbf{x}^{(j)}$, extracting features that are invariant to distant content. Weight sharing is achieved by applying the identical kernel $\mathbf{w}_{k,r}$ to process both patches $\mathbf{x}^{(1)}$ and $\mathbf{x}^{(2)}$. This spatial reuse of parameters is mathematically equivalent to a discrete convolution with a stride equal to the patch size.
\end{remark}

\textbf{Training objective.}
Given a training dataset $\mathcal{S} = \{(\mathbf{x}_i,y_i)\}_{i=1}^n$ drawn from the distribution $\mathcal{D}$, we train the neural network by minimizing the empirical risk using the cross-entropy loss function:
\begin{equation}
	\begin{aligned}
		\mathcal{L}_\mathcal{S} =& \frac{1}{n}\sum_{i=1}^{n}\mathcal{L}(\mathbf{W},\mathbf{x}_i,y_i) = \frac{1}{n}\sum_{i=1}^{n} [-\log(\text{prob}_{y_i}(\mathbf{W},\mathbf{x}_i))],
	\end{aligned}
\end{equation}
where $\text{prob}(\cdot)$ denotes the softmax predictions with the output of the neural network.

\textbf{Differential privacy and training algorithm.}
\ac{dp} \citep{dwork2014algorithmic} (defined below) serves as a benchmark for quantifying privacy leakage, providing rigorous privacy guarantees.

\begin{definition}[$(\epsilon, \alpha)$-Differential privacy]\label{definition: dp}
	A randomized algorithm $\mathcal{M}: \mathcal{Z} \rightarrow \mathcal{R}$ is ($\epsilon,\alpha$)-\ac{dp} if, for every pair of neighboring datasets $Z, Z' \in \mathcal{Z}$ that differ by one entry, and for any subset of outputs $\mathcal{P}\subseteq\mathcal{R}$, the following holds: $\mathbb{P}\left[\mathcal{M}(Z)\in\mathcal{P}\right]\le e^{\epsilon} \mathbb{P}\left[\mathcal{M}(Z')\in\mathcal{P}\right]+\alpha$.
\end{definition}

\ac{dpsgd}, a standard and widely used training algorithm \citep{abadi2016deep} (details in Appendix~\ref{appendix: dpsgd}), adds privacy-preserving noise to the gradients during optimization and performs gradient clipping. The update rule for the model parameters is given by:
\begin{equation}
	\begin{aligned}    \label{equ: dp update}
		\mathbf{W}^{(t+1)} \!=& \mathbf{W}^{(t)}\! - \!\frac{\eta}{B}\cdot\!\sum_{(\mathbf{x},y)\in\mathcal{S}^{(t)}}\!\! \text{clip}_C\left(\nabla \mathcal{L}\left(\mathbf{W}^{(t)},\mathbf{x},y\right)\right)
		\\
        &+\! \eta\cdot\mathbf{n}^{(t)},  
	\end{aligned}
\end{equation}
where $\eta$ is the learning rate, $\mathcal{S}^{(t)}$ represents a mini-batch of size $B$ randomly selected at iteration $t$, $\mathbf{n}^{(t)}$ is the noise added for privacy protection, sampled from $\mathcal{N}(0, \sigma_n^2\mathbf{I})$, and $\text{clip}_C(\mathbf{x})$ is the gradient clipping function with a clipping threshold $C$ on vector $\mathbf{x}$:	$\text{clip}_C(\mathbf{x}) = \frac{\mathbf{x}}{\max\left\{1,\left\|\mathbf{x}\right\|_2/{C}\right\}}$. 
We initialize network parameters $\mathbf{W}^{(0)}$ with i.i.d. Gaussian distributions $\mathcal{N}(0,\sigma_0^2\mathbf{I})$.

\section{Test loss analysis}\label{sec: test loss}
In this section, we analyze the test loss of \ac{dpsgd}-trained \acp{cnn}.
This analysis serves as a foundational step for examining the impacts of DP-SGD in Section \ref{sec: implications}.

\subsection{Preliminary}
We begin by defining the standard test loss of a trained model on the data distribution $\mathcal{D}_{i,j}$, for any $i\in\{1, 2\}$ and $j\in\{\text{maj},\text{min}\}$ as follows: $\mathcal{L}_{\mathcal{D}_{i,j}}(\mathbf{W}) = \mathbb{E}_{(\mathbf{x},y)\sim \mathcal{D}_{i,j}}\left[\mathcal{L}(\mathbf{W},\mathbf{x},y)\right]$.
As is common in the literature (e.g., \citep{allen2019convergence,kou2023benign}), our analysis relies on the following conditions and assumptions:

\begin{condition}\label{condition}
	Suppose that there exists positive constants $0<c_1\le1$, and sufficiently large constants $c_2,c_3,c_4>0$. For certain probability parameter $\delta>0$, the following hold: \begin{enumerate}
     \item The data dimension satisfies $d\ge c_2\log(n/\delta)$.
     \item The batch size $B$ is proportional to the training dataset size $n$, specifically $B\ge c_1\cdot n$, enabling stochastic gradients to effectively leverage large datasets.
     \item The feature size, patch noise, and \ac{dpsgd} noise satisfy $\left\|\mathbf{u}_{i,j}\right\|_2 \ge c_3  \sigma_p \ge c_3\sigma_n$, for $i\in[2]$ and $ j\in\{\text{maj},\text{min}\}$.
     \item The learning rate satisfies $\eta \le \left(c_4(C+\sqrt{d}\sigma_n)(\max_{i,j}\left\|\mathbf{u}_{i,j}\right\|_2 + \sqrt{d}\sigma_p)\right)^{-1}$.
 \end{enumerate}
\end{condition}
The first condition on the data dimension guarantees that the squared norm of the noise concentrates around a scaling of $\sigma_p^2d$ with high probability.
The second condition of a lower bound on batch size ensures that the variance of the mini-batch gradient decreases with dataset size $n$, thereby preserving the classical $\mathcal{O}(1/\sqrt{n})$ generalization benefit even under DP noise.
The third condition on feature magnitude, patch noise, and DP noise guarantees that the signal-to-noise ratio remains sufficient for DP-SGD to keep reducing the training loss throughout optimization. These constraints are mild and hold in standard DP training regimes, where the privacy noise is not deliberately set to be overwhelming and is averaged over the batch size.
The fourth condition on learning rate ensures that each parameter update remains bounded with high probability, a requirement that naturally holds in practice.

\begin{assumption}[$s$-non-perfect model]\label{assumption: non-perfect}
	We assume that the model is almost surely not perfect on any test example, i.e., for some constant $s>0$,	$\mathcal{L}\left(\mathbf{W}^{(t)}, \mathbf{x},y\right) \ge s$, for all $(\mathbf{x},y) \sim \mathcal{D}$.
\end{assumption}
Assumption \ref{assumption: non-perfect} is a mild one, particularly given the inherent randomness introduced by \ac{dpsgd} in the model training process. 
Consequently, the trained model is stochastic and unlikely to achieve zero cross-entropy loss on any given test example.
Next, we define the following terms to facilitate our subsequent analyses:
\begin{definition}  \label{def: fnr}
    We define \acp{fnr}, clipping factor, and expected proportions as follows:
    \begin{enumerate}
        \item The feature-to-noise ratio for class $i$ in group $j$ is defined as: $\mathcal{F}_{i,j} = \frac{\left\|\mathbf{u}_{i,j}\right\|_2}{\sigma_n}$.

        \item The clipping factor for class $i$ in group $j$, which quantifies the maximum change in gradient magnitude for data from class $i$ and group $j$, is defined as: $\Lambda_{i,j} = \frac{C}{\left\|\mathbf{u}_{i,j}\right\|_2+\sigma_p\sqrt{d}}$.

        \item The expected proportion of class $i$ group $j$ data in the whole training dataset is defined as $\gamma_{i,j}$. Specifically, $\gamma_{1,\text{maj}} = p_cp_f, \gamma_{1,\text{min}} = p_c(1-p_f), \gamma_{2,\text{maj}} = (1-p_c)p_f, \gamma_{2,\text{min}} = (1-p_c)(1-p_f)$.
    \end{enumerate}
\end{definition}

\subsection{Standard test loss analysis}
Building upon the aforementioned conditions, assumption, and definitions, we characterize the upper bound for the test loss of \ac{dpsgd} trained models in Theorem \ref{theorem: convergence}.

\begin{theorem}\label{theorem: convergence}
	Under Condition \ref{condition} and Assumption \ref{assumption: non-perfect}, with a probability at least $1-\delta$, for any $i\in \{1,2\}, j\in \{\text{maj},\text{min}\}$, the test loss of a \ac{dpsgd} trained model satisfies:
			$\mathcal{L}_{\mathcal{D}_{i,j}}\!\left(\!\mathbf{W}^{(T)}\!\right)
			\le \bar{L}_{i,j}(\mathbf{W}^{(T)})$,
	where 
	\begin{equation}\label{equ: L bar}
		\begin{aligned}
			\bar{L}_{i,j}(\mathbf{W}^{(T)})\!=&\! 
            \underbrace{\exp\!\left(\!-\Omega\left(\!\frac{\Lambda_{i,j}\gamma_{i,j}\left\|\mathbf{u}_{i,j}\right\|_2^2}{m}T\right)\!\!\right)\!\mathcal{L}_{\mathcal{D}_{i,j}}\!\left(\!\mathbf{W}^{(0)}\!\right)}_\textnormal{Vanishing error}
            \\\!+&\!\underbrace{\mathcal{O}\!\left(\!\frac{1}{\sqrt{n}\gamma_{i,j}\Lambda_{i,j}}\!\right)}_\textnormal{Generalization error}\!+\!\underbrace{\mathcal{O}\!\left(\!\frac{m}{\Lambda_{i,j}\gamma_{i,j}\mathcal{F}_{i,j}}\!\right)}_\textnormal{Privacy protection error}\!.
		\end{aligned}
	\end{equation}
\end{theorem}
Due to the space limit, we provide a proof sketch in Appendix~\ref{app: sketch} and defer the proof of Theorem~\ref{theorem: convergence} to Appendix~\ref{appendix: theorem convergence proof}.
Given that the stochasticity of the noise can sometimes improve model performance, the high-probability test loss lower bound becomes less informative. 
Instead, we provide a lower bound for the expected test loss in the following theorem:

\begin{theorem}\label{theorem: lower bound}
	Under Condition \ref{condition} and Assumption \ref{assumption: non-perfect}, with the number of iterations $T\ge \Omega\left(-\frac{1}{\log\left(1-\Omega\left(\frac{\eta\min_{i,j}\{\gamma_{i,j}\left\|\mathbf{u}_{i,j}\right\|_2^2\}}{m}\right)\right)}\right)$ and a probability at least $1-\delta$, for any $i\in\{1,2\}$ and $j\in\{\text{maj},\text{min}\}$, the expected test loss of a \ac{dpsgd} trained model satisfies:
	\begin{equation}\nonumber
		\begin{aligned}
			\mathbb{E}\!\left[\mathcal{L}_{\mathcal{D}_{i,j}}(\mathbf{W}^{(T)})\right]\! \ge& \exp\!\left(\!-\Omega\left(\frac{\gamma_{i,j}\left\|\mathbf{u}_{i,j}\right\|_2^2}{m}T\right)\!\!\right)\!\mathcal{L}_{\mathcal{D}_{i,j}}(\mathbf{W}^{(0)})\!
			\\
            &+\! \underbrace{\Omega\!\left(\!\frac{d\sigma_p^2}{\gamma_{i,j}\mathcal{F}_{i,j}^2}\!\right)}_\textnormal{Privacy protection error}\! -\mathcal{O}\!\left(\!\frac{1}{\gamma_{i,j}}\sqrt{\frac{1}{n}}\right).
		\end{aligned}
	\end{equation} 
\end{theorem}

Theorems \ref{theorem: convergence} and \ref{theorem: lower bound} demonstrate that both the upper and lower bounds for the test loss are inversely related to the feature-to-noise ratio  $\mathcal{F}_{i,j}$.
Moreover, Theorem \ref{theorem: convergence} highlights the presence of non-vanishing error terms in the test loss bound: the \textit{generalization error} and \textit{the privacy protection error}. Specifically, 
\textbf{Generalization error} arises due to the noise present in the data patches.
This error decreases at a rate of $\mathcal{O}(\frac{1}{\sqrt{n}})$, which is consistent with generalization error bounds established for neural networks \citep{arora2019fine, xu2012robustness}. 
\textbf{Privacy protection error} originates from the \ac{dp} noise injected during training. For a fixed privacy budget ($(\epsilon,\alpha)$ in Definition \ref{definition: dp}), privacy composition theorems \citep{wang2019subsampled,mironov2017renyi} establish that the required noise variance scales as $\sigma_n^2 = \Theta(T)$ \citep{wang2019subsampled,mironov2017renyi}.
Consequently, the error introduced by privacy protection grows with the number of iterations at a rate of  $\mathcal{O}(\sqrt{T})$, aligning with the non-vanishing training error observed for \ac{dpsgd} on Lipschitz-smooth objectives \citep{bu2024automatic}.

\begin{figure}
    \centering
    \includegraphics[width=0.9\linewidth]{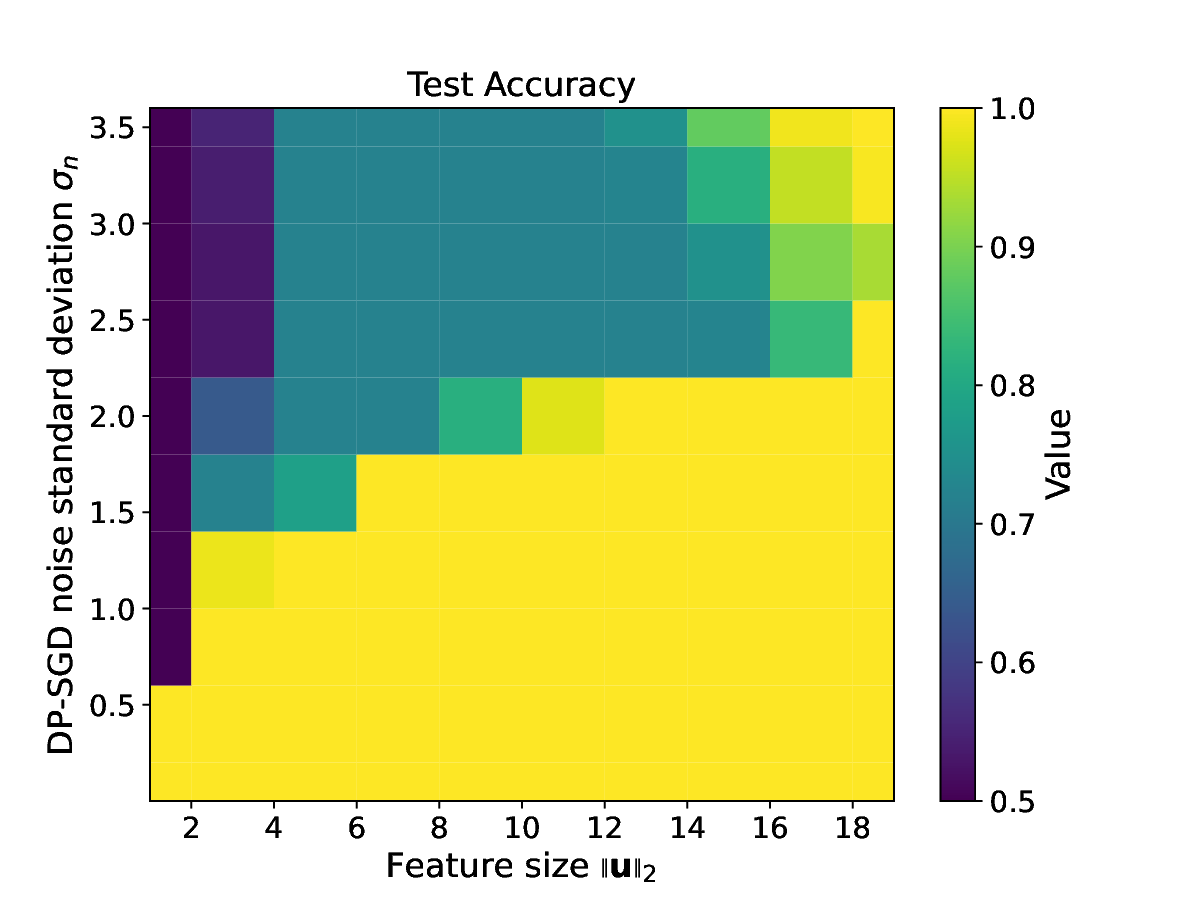}
	\caption{Illustration of the privacy-utility phase transition between benign and harmful privacy protection. 
        }
	\label{fig:phase}
\end{figure}

\textbf{Privacy-utility tradeoff.}
While the aforementioned theorems characterize a continuous degradation in the test loss as the privacy protection error accumulates, this gradual accumulation manifests quite differently when evaluating discrete classification performance. Due to the discontinuity of the 0-1 accuracy, our analysis reveals a sharp phase transition where privacy protection suddenly shifts from benign to harmful to model performance (see Figure \ref{fig:phase}). Simulation details are provided in Appendix~\ref{app: figure1}.

\subsection{Adversarial test loss analysis}\label{sec: adversarial test loss}
In this subsection, we analyze the impact of \ac{dpsgd} on adversarial robustness.
This analysis serves as a foundation for understanding the side effect of reduced adversarial robustness, which is further discussed in Section \ref{sec: implications}.

Adversarial robustness refers to a model's ability to maintain its predictive accuracy when subjected to carefully crafted input samples, commonly known as adversarial examples.

\begin{definition}[Adversarial example]
	For a given data example $(\mathbf{x},y)$, an adversarial example is $\tilde{\mathbf{x}} = \mathbf{x} + \boldsymbol{\zeta}$, where $\boldsymbol{\zeta} = \arg\max_{\left\|\boldsymbol{\zeta}\right\|_p \le \bar{\zeta}} \mathcal{L}\left(\mathbf{W},\mathbf{x} +\boldsymbol{\zeta}, y\right)$ is the adversarial perturbation and $\bar{\zeta}$ is the perturbation radius.
\end{definition}

To quantify the adversarial robustness of a trained model, we employ the adversarial test loss, defined as follows:	$\mathcal{L}^\text{adv}_\mathcal{D}(\mathbf{W}) = \mathbb{E}_{(\mathbf{x},y)\sim \mathcal{D}}\left[\max_{\left\|\boldsymbol{\zeta}\right\|_p \le \bar{\zeta} }\mathcal{L}(\mathbf{W},\mathbf{x}+\boldsymbol{\zeta},y)\right]$.

We present our main result on adversarial test loss in the following theorem.

\begin{theorem}\label{theorem: convergence_adv}
	Under Condition \ref{condition} and Assumption \ref{assumption: non-perfect}, with a probability at least $1-\delta$, for any $i \in \{1,2\}$ and $ j\in\{\text{maj},\text{min}\}$, the adversarial test loss of a \ac{dpsgd} trained model satisfies:
	\begin{equation}
		\begin{aligned}
			\mathcal{L}^\textnormal{adv}_{\mathcal{D}_{i,j}}\left(\mathbf{W}^{(T)}\right)\!\le&\bar{L}_{i,j}(\mathbf{W}^{(T)})\!\\
			&+\!\underbrace{\mathcal{O}\!\left(\left[\frac{T}{m} C\!+\!\frac{\sqrt{Td}}{m}\sigma_n\!+\!\sqrt{d}\sigma_0\right]\!\bar{\zeta} d^{1-\frac{1}{p}}\right)}_\textnormal{By adversarial perturbation}.
		\end{aligned}
	\end{equation}
\end{theorem}
We defer the proof of Theorem~\ref{theorem: convergence_adv} to Appendix~\ref{appendix: theorem convergence_adv proof}.
Theorem \ref{theorem: convergence_adv} indicates that the error induced by adversarial perturbation increases at a rate of  $\mathcal{O}(T)$,
% Although \ac{dpsgd} does not apply any adversarial training components, 
with a \ac{dp} noise term $\frac{\sqrt{Td}}{m}\sigma_n$, which increases at a rate of $\mathcal{O}(\sqrt{T})$.
This suggests that introducing privacy-preserving noise can exacerbate the model's susceptibility to adversarial attacks, thereby worsening its adversarial robustness.
Furthermore, the derived test loss bound is consistent with the excess risk bounds observed in adversarial training scenarios \citep{xiao2022stability}.

\section{Understanding DP-SGD impacts}\label{sec: implications}
In this section, we leverage the derived test loss bounds to interpret the phenomena observed in \ac{dpsgd}-trained models.

\subsection{Interpretation of disparate impact}
To understand the disparate impact, we first define the distribution of class $i$ as $\mathbb{P}_{\mathcal{D}_{i}}\!\left[(\mathbf{x},y)\right] \!=\! \mathbb{P}_{\mathcal{D}}\!\left[(\mathbf{x},y)| y=i\right],\ i \in\{1,2\}$, and evaluate the model’s performance on different classes by bounding the test loss on data from each class.

\begin{corollary}[Disparate impact of different classes]\label{corollary: disparate impact}
	Under Condition \ref{condition} and Assumption \ref{assumption: non-perfect}, with a probability at least $1-\delta$, for any $i \in \{1,2\}$, the test loss of a \ac{dpsgd}-trained model satisfies:	
    $\mathcal{L}_{\mathcal{D}_{i}}\!\left(\!\mathbf{W}^{(T)}\!\right)
			\le\frac{1}{\sum_{j\in\{\textnormal{min},\textnormal{maj}\}}\!\gamma_{i,j}}\sum_{j\in\{\textnormal{min},\textnormal{maj}\}}\!\gamma_{i,j}\cdot\bar{L}_{i,j}\!\left(\mathbf{W}^{(T)}\right)$.
\end{corollary}

Similarly, we define the distribution of group $j$ as $\mathcal{D}_j, j\in\{\text{maj},\text{min}\}$ and evaluate the model performance across different groups (majority and minority) by examining the test loss on data distribution within each group.

\begin{corollary}[Disparate impact of subpopulation groups]\label{corollary: unfairness}
	Under Condition \ref{condition} and Assumption \ref{assumption: non-perfect}, with probability at least $1-\delta$, for any $j\in\{\textnormal{maj},\textnormal{min}\}$, the test loss of a \ac{dpsgd} trained model satisfies:
	$\mathcal{L}_{\mathcal{D}_{j}}\!\left(\!\mathbf{W}^{(T)}\!\right)\le \frac{1}{\sum_{i=1}^{2}\gamma_{i,j}}\sum_{i=1}^{2}\gamma_{i,j}\cdot\bar{L}_{i,j}\left(\mathbf{W}^{(T)}\right)$.
\end{corollary}

Recalling the expression for $\bar{L}$ in Eq.~\eqref{equ: L bar} and the terms in Definition \ref{def: fnr}, Corollaries \ref{corollary: disparate impact} reveal three primary sources of disparate impact: 
\textit{1) Feature disparity} $\left\|\mathbf{u}_{i,j}\right\|_2$,
\textit{2) Gradient clipping} $\Lambda_{i,j}$,
and \textit{3) Data imbalance} $\gamma_{i,j}$.
We discuss them in detail below.

\textbf{Feature disparity.}
The feature-to-noise ratio $\mathcal{F}_{i,j}$ depends on the feature sizes of the data $\mathbf{u}_{i,j}$. 
In real-world applications, data from different classes or groups may exhibit significantly different feature sizes, resulting in divergent model performance. 
Our results show that the most frequently misclassified data samples are those with long-tailed features, which is verified in Section \ref{exp: real}.

\textbf{Gradient clipping.}
Theorem \ref{theorem: convergence} indicates that the test loss is influenced by the clipping factor $\Lambda_{i,j}$. 
Classes or groups with larger gradient norms will experience more aggressive clipping, leading to poorer feature learning performance.

\textbf{Data imbalance.}
Theorems \ref{theorem: convergence} and \ref{theorem: lower bound} demonstrate that the privacy protection error decreases as the data proportion $\gamma_{i,j}$ increases. Consequently, groups or classes with larger data representations achieve superior model utility, whereas underrepresented groups incur disproportionately higher error rates.
This raises a risk of worse model performance on the skewed data sources\footnote{Take ImageNet \citep{deng2009imagenet} as an example. More than 45\% of ImageNet data comes from the United States, corresponding to only 4\% of the world’s population;
In contrast, China and India contribute just 3\% of ImageNet data \citep{zou2018design}.
Thus, \ac{dpsgd} trained models on ImageNet may perform poorly on tasks related to China and India.}
and the long-tailed distributed applications \citep{feldman2020neural}.\footnote{For example, in the SUN dataset \citep{xiao2010sun}, the number of examples in each class displays a long-tailed structure \citep{feldman2020does}.}

\subsection{Interpretation of adversarial robustness}
As shown in Theorem \ref{theorem: convergence_adv}, 
\ac{dpsgd} tends to result in higher adversarial test loss.
We interpret this decrease in adversarial robustness from two perspectives.

\textbf{Feature learning.} 
As noted in \cite{allen2020towards}, an adversarially robust model typically removes class-irrelevant non-robust noise and learns robust features.
However, \ac{dpsgd} injects significant noise during training, leading neural networks to inevitably learn non-robust class-irrelevant noise. 

\textbf{Network parameter growth.} 
Due to the noise added by \ac{dpsgd}, the network parameters’ norms increase as the number of iterations grows. 
Adversarial perturbations $\boldsymbol{\zeta}$ attack the model by changing the activated inner products of neurons, i.e.,
\begin{equation}
	\begin{aligned}
		&F_k(\mathbf{W},\mathbf{x}+\boldsymbol{\zeta})\\
        =&\frac{1}{m}\sum_{r=1}^{m}\sum_{j=1}^{2}\left[\sigma\left(\left\langle\mathbf{w}_{k,r},\mathbf{x}^{(j)}\right\rangle +\!\!\!\underbrace{\left\langle\mathbf{w}_{k,r}, \boldsymbol{\zeta}^{(j)}\right\rangle}_\text{By adversarial perturbation}\right)\right]\!,
	\end{aligned}
\end{equation}
where higher network parameter norms result in increased vulnerability to adversarial attacks.

\subsection{Public-pretraining and private-finetuning}
As demonstrated in \cite{berrada2023unlocking}, pretraining on public datasets can significantly mitigate the accuracy drop and reduce the side effects caused by \ac{dpsgd}. 
The authors claimed that with pretraining, the neural network utilizes good pretrained features, and the loss at initialization $\mathcal{L}_{\mathcal{D}_{i,j}}(\mathbf{W}^{(0)})$ is relatively small, which reduces the number of finetuning iterations needed, leading to smaller privacy protection errors and thus, milder side effects.

However, in \cite{tramer2022position}, the authors pointed out that many existing pretraining and finetuning datasets share a similar data distribution (e.g., pretraining on ImageNet and finetuning on CIFAR-10). 
In this subsection, we explore how shifts in distribution between pretraining and finetuning data affect private finetuning performance.

\textbf{Setup.} 
Consider a two-layer CNN model (defined in Section \ref{sec: model}) trained on a simplified data distribution $\mathcal{D}_\text{pt}$ where each class has a single feature with the same magnitude: $\mathbf{u}_1$ for class 1 and $\mathbf{u}_2$ for class 2, with$\left\|\mathbf{u}_1\right\|_2 = \left\|\mathbf{u}_2\right\|_2$. The model is trained using standard SGD and then fine-tuned on a similar data distribution $\mathcal{D}_\text{ft}$, where the data is generated from $\mathcal{D}_\text{pt}$ and then rotated by an angle $\theta$ (Details are discussed in Appendix~\ref{appendix: pt-ft}).  
Based on the results in \cite{kou2023benign}, a ReLU \ac{cnn} trained with gradient descent learns the feature in a constant order. 
Therefore, for simplicity, we assume that the pre-trained model parameters are $\tilde{\mathbf{w}}_{j,r} = C_1 \cdot\mathbf{u}_{j} + C_3\cdot\mathcal{N}(0,\sigma_p\mathbf{H})$, where $C_1$ and $C_3$ are constants.

\begin{condition}\label{condition2}
	Suppose that there exist a positive constant $0< c_5 \le 1$ and sufficiently large constants $c_6,c_7,c_8>0$. For all $i\in[2]$,
 \begin{enumerate}
    \item The data dimension satisfies $d\ge c_6\log(n/\delta)$.
     \item The batch size $B$ is proportional to the training dataset size $n$, specifically $B\ge c_5\cdot n$, enabling stochastic gradients to effectively leverage large datasets.
     \item The feature size, patch noise, and \ac{dpsgd} noise satisfy $\left\|\mathbf{u}_{i}\right\|_2 \ge c_7\sigma_p \ge c_7 \sigma_n$.
     \item The learning rate satisfies $\eta \le \left(c_8(C+\sqrt{d}\sigma_n)(\max_{i}\left\|\mathbf{u}_{i}\right\|_2 + \sqrt{d}\sigma_p)\right)^{-1}$.
 \end{enumerate}
\end{condition}

Define $\Lambda_i$ in a similar manner that $\Lambda_{i} = \frac{C}{\left\|\mathbf{u}_{i}\right\|_2+\sigma_p\sqrt{d}}$.
We characterize the finetuning test loss as follows.
\begin{proposition}\label{proposition: pt ft}
	Under Condition \ref{condition2} and Assumption \ref{assumption: non-perfect}, for any $i\in\{1,2\}$, with probability $1-\delta$, the test loss of private fine-tuned model with satisfies:
	\begin{equation}
		\begin{aligned}
			\mathcal{L}_{\mathcal{D}_\textnormal{ft}}(\mathbf{W}^{(T)})\le& 
			\exp\!\left(\!-\Omega\left(\!\frac{\Lambda_i\left\|\mathbf{u}_{i}\right\|_2^2}{m}T\right)\!\right)\!\cdot\tilde{L}\\
			&+\mathcal{O}\left(\frac{\sqrt{d}}{\sqrt{n}\Lambda_{i}}\right)+\mathcal{O}\left(\frac{m\sqrt{d}\sigma_n}{\Lambda_i\left\|\mathbf{u}_i\right\|_2}\right),
		\end{aligned}
	\end{equation}
	where 
	\begin{equation}
		\begin{aligned}
			&\tilde{L} = -\frac{1}{2}\ln\left(\frac{\exp(C_1 \cos\theta \left\|\mathbf{u}_{2}\right\|_2^2)}{\exp(C_1 \cos\theta \left\|\mathbf{u}_{2}\right\|_2^2)+\exp( C_3\sigma_p^2)}\right)\\
			 &-\!\!\frac{1}{2}\!\ln\!\!\left(\!\frac{\exp(C_1 \cos\theta \left\|\mathbf{u}_{1}\right\|_2^2)}{\exp(C_1\! \cos\theta \left\|\mathbf{u}_{1}\right\|_2^2)\!+\!\exp(C_1\!\sin\theta \left\|\mathbf{u}_{1}\right\|_2^2\!+\! C_3\sigma_p^2)}\!\!\right).\\
		\end{aligned}
	\end{equation}
\end{proposition}

	A worth noting fact is that $\tilde{L}$ increases with $\theta$, meaning that the private finetuning test loss increases as the feature difference between the pretraining and finetuning distributions $\theta$ increases. 
	Even worse, if $\tilde{L}> \mathcal{L}_{\mathcal{D}_2}(\mathbf{W}^{(0)})$, pretraining can lead to worse performance than training from scratch.

\section{Solutions for improving feature-to-noise ratio}
Motivated by our preceding analyses, this section introduces several strategies to enhance the feature-to-noise ratio. Empirical validations confirming the efficacy of these approaches are deferred to Appendices~\ref{appendix: freezing} and \ref{appendix: data augmentation}.

\textbf{Data augmentation.}
Data augmentation is a well-established technique for cultivating robust feature representations in neural networks \cite{shen2022data}. This is particularly beneficial in private learning, as the enriched task-relevant information provided by augmentation helps amplify the FNR, yielding models with significantly higher utility and accuracy \cite{de2022unlocking}.

\textbf{Network freezing/pruning.}
Another approach to reducing the \ac{fnr} is network freezing or pruning \cite{frankle2018lottery}. The underlying principle is that neurons in a trained network, shaped by non-linear activation functions, exhibit functional specialization and thus contribute unequally to the final prediction. Consequently, selectively freezing or pruning neurons with low contribution can enhance the model's focus on salient features, leading to significant reductions in \ac{fnr}. Building on this insight, we propose a stage-wise network freezing method that effectively improves model performance, with a detailed description provided in Appendix~\ref{appendix: freezing}.

\section{Experiments}   \label{sec: experiment}
\subsection{Synthetic datasets}  
\textbf{Synthetic data generation.}
We generate synthetic data following the data distribution described in Section \ref{sec: model}. 
Specifically, we set both the training and test dataset sizes to $450$.
We set the feature vector length in each patch to $100$. 
The feature vector sizes and dataset sizes of the majority and minority groups of classes $0$ and $1$ are specified as $\left\|\mathbf{u}_{1,\text{maj}}\right\|_2 = 4, \gamma_{1,\textnormal{maj}} = 44\%, \left\|\mathbf{u}_{1,\text{min}}\right\|_2 = 2, 
\gamma_{1,\textnormal{min}} = 22\%,
\left\|\mathbf{u}_{2,\text{maj}}\right\|_2 = 1.5, 
\gamma_{2,\textnormal{maj}} = 22\%
\text{ and }\left\|\mathbf{u}_{2,\text{min}}\right\|_2 = 0.5, \gamma_{2,\textnormal{min}} = 11\%$.
These choices of feature vector sizes and dataset proportions introduce feature disparity and data imbalance into the synthetic data.
In addition, we set the standard deviation of the noise patch to $\sigma_p = 0.2$.

We train a two-layer \ac{cnn} with ReLU activation function and cross-entropy loss (see Section \ref{sec: model}).
The number of neurons is $64$, i.e., $m=32$. 
We use the default PyTorch initialization and train the \ac{cnn} with \ac{dpsgd}, with a batch size of $B = 128$ for $20$ epochs.

\textbf{Test loss.} 
We evaluate the test loss of the trained model across different groups and features in Figure \ref{fig: nu-testloss}(a).
We observe that the test loss generally increases with the \ac{dp} noise standard deviation, aligning with the findings in Theorem~\ref{theorem: convergence}, which suggests that the upper bound on the test loss depends on the corresponding \ac{fnr}.
Furthermore, the class and group with larger feature sizes incur smaller test losses.
Notably, the gaps among the groups become more significant as the noise standard deviation increases.

\vspace{-5pt}
\begin{figure}[ht]
	\centering
        \subfigure[Test loss versus DP noise standard deviation $\sigma_n$.]
	{\includegraphics[width=0.8\linewidth]{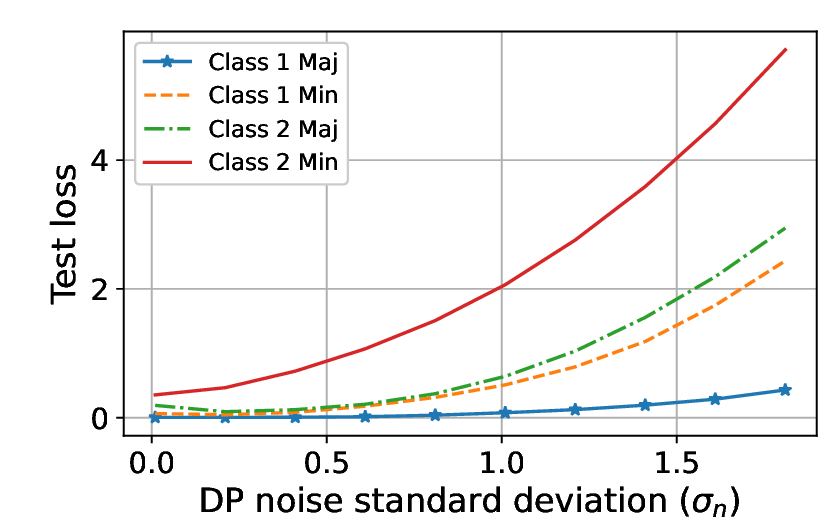}}
	% \label{fig: testloss}
	\subfigure[Adversarial test loss attacked with the Projected Gradient Descent Method.]
        {\includegraphics[width=0.8\linewidth]{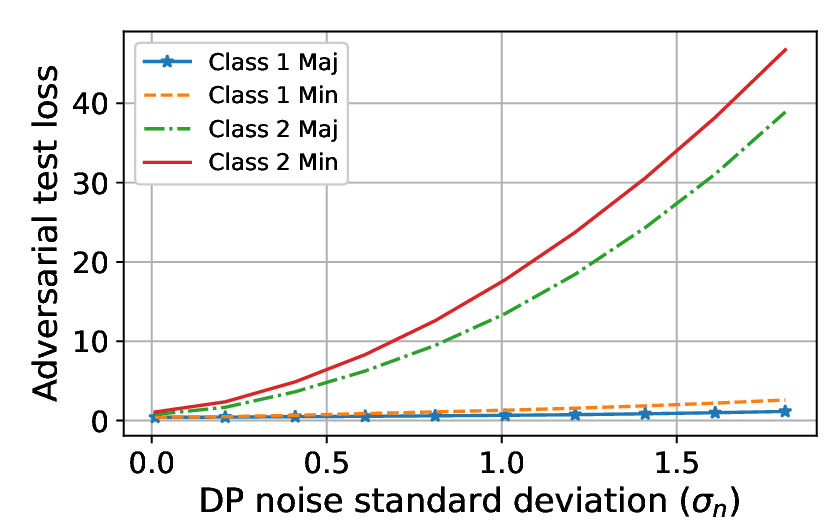}}
        % \label{fig: test_loss_adv}
        \caption{Model standard test loss and adversarial test loss.}
        \label{fig: nu-testloss}
\end{figure}

\textbf{Adversarial robustness.}
In Figure \ref{fig: nu-testloss}(b), we assess the adversarial robustness by attacking the model with the projected gradient descent method (generate the adversarial examples by maximizing the loss with projected gradient descent) with $\bar{\zeta} = 0.02$. We observe that the \ac{dpsgd} trained model degrades in adversarial robustness for certain groups/classes.
This aligns with the results from Theorem \ref{theorem: convergence} that the upper bound of the model's adversarial test loss increases with \ac{dp} noise standard deviation.

\subsection{Real-world datasets}\label{exp: real}
\textbf{Setup.}
For MNIST \cite{lecun1998gradient} and CIFAR-10 \cite{krizhevsky2009learning}, we train LeNet and a \ac{cnn} following the architecture in \citep{tramer2020differentially} with \ac{dpsgd}.
We fix the privacy budget as $\epsilon = 3, \alpha = 10^{-5}$, the gradient clipping threshold as $C = 0.1$ \citep{tramer2020differentially}, the batch size as 256, and try various learning rates.
We use the \ac{dpsgd} implementation in Opacus \citep{yousefpour2021opacus}.
We generate adversarial examples using the projected gradient descent method with $\bar{\zeta} = 4/255$.

\begin{figure}[ht]
	\centering
	\subfigure{\includegraphics[width=0.49\linewidth]{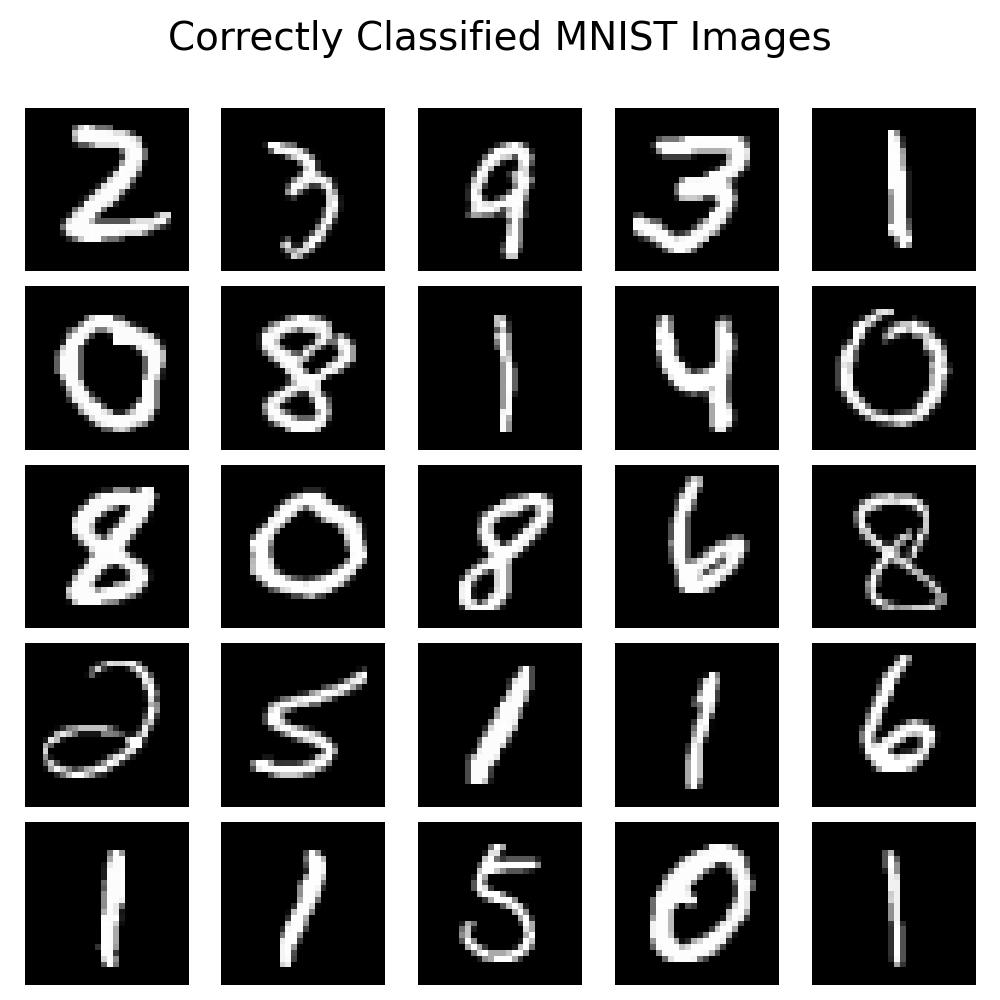}}
        \subfigure{\includegraphics[width=0.49\linewidth]{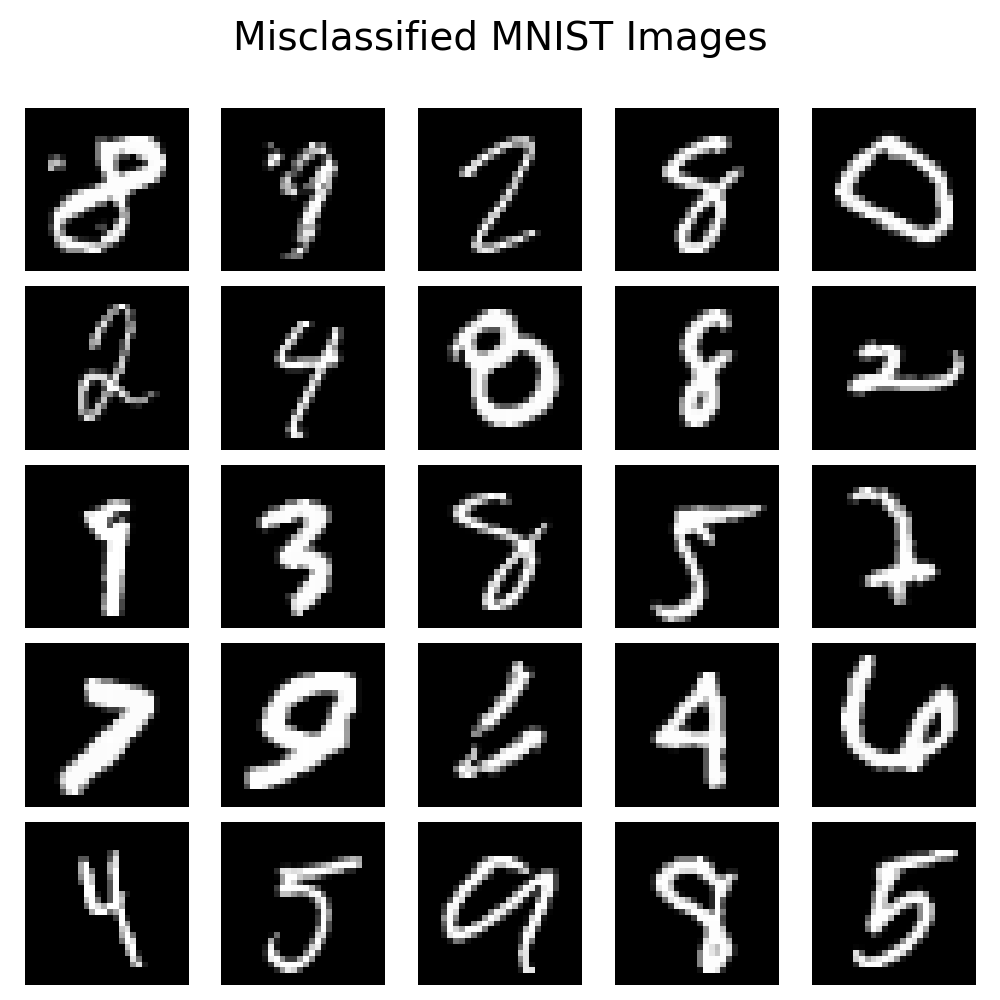}}
        \caption{Visualization of correctly classified (left) and misclassified (right) images.}
        \label{fig: testloss}
\end{figure}

\textbf{Impact of feature size.}
We first visualize the correctly classified and misclassified images for MNIST\footnote{MNIST's features carry clear physical meaning that long-tailed data corresponds to poorly written digits~\citep{feldman2020neural}.} in Figure~\ref{fig: testloss}.
The correctly classified images are mostly well-written digits, while the misclassified images contain more poorly written ones. 
This verifies our theory that DP-SGD hurts more on long-tailed data even within the same class.

In real images, identifying which specific parts represent features can be challenging. Therefore, in our experiment, we treat the object within the image as the "feature" and the background as "noise."
To emulate the image backgrounds, we apply padding to the periphery of the input images and then resize them back to the original size (as shown in Figure~\ref{fig: imagepadding} in the appendix).
Intuitively, a higher padding ratio corresponds to a lower FNR.
The table below summarizes test and adversarial test accuracy for different padding ratios (we also provide results for vision transformers on padded images in Appendix~\ref{app: vit}). As indicated in our theory, the model accuracy decreases with the padding ratios.
We also validate the impact of class sizes in Appendix~\ref{app: impact_size}.

\begin{table}[h]
	%\caption{Sample table title}
	\label{sample-table1}
	\begin{center}
		\begin{tabular}{llllll}
        \toprule
			Padding ratio  & 0\%  & 26\% &50\% & 62\% & 75\%\\
			\midrule
			MNIST &97\% &97\% & 95\%& 94\% & 86\%\\
			CIFAR-10 & 58\% & 54\%  & 51\% & 48\% &46\%\\
			MNIST (adv) &95\% &89\% & 50\%& 20\% & 1\%\\
			CIFAR-10 (adv) & 3\%& 2\% & 1\% & 0\% &0\%\\
            \bottomrule
		\end{tabular}
	\end{center}
\end{table}

\textbf{Impact of public-pretraining. }
We examine the impact of feature differences, introduced by rotation angles, on the performance of the public-pretraining and private-finetuning paradigms.
Keeping the training dataset unrotated, we rotate the test images and split them into a finetuning training dataset and a test dataset.
We use both \ac{cnn} and ResNet-18 \citep{he2016deep} architectures for CIFAR-10. 
The table below shows the private-finetuning test performance under different rotation angles. As predicted by our theory, the model accuracy decreases with the rotation angle. 

\begin{table}[h]
	%\caption{Sample table title}
	\label{sample-table2}
	\begin{center}
		\begin{tabular}{lllll}
        \toprule
			Rotation angle  & $0^\circ$  & $22.5^\circ$ & $45^\circ$ & $67.5^\circ$\\
			\midrule
			MNIST & 99\% &  97\% & 94\% & 94\%\\
			CIFAR-10 (CNN)\!&  70\% & 59\% & 51\% &  49\% \\
			CIFAR-10 (ResNet)\!\! & 91\%& 66\%& 43\% & 37\%\\
            \bottomrule
		\end{tabular}
	\end{center}
\end{table}

\section{Conclusions and future works}\label{sec: conclusions}
In this paper, we investigate the side effects of \ac{dpsgd} in two-layer ReLU \acp{cnn}, revealing that these side effects depend on the data's feature, data noise, and privacy-preserving noise.
Our results uncover three sources of disparate impact: \textit{gradient clipping}, \textit{data imbalance}, and \textit{feature disparity}.
In addition, we show that the privacy-preserving noise introduces randomness into the learned features, leading to worse adversarial robustness. 
We also show that finetuning performance with pre-trained models deteriorates as the feature differences between the pretraining and finetuning datasets increase.
Numerical results on both synthetic and real-world datasets validate our theoretical analyses.
Larger neural networks may involve more complex learning dynamics than the one we analyzed. 
Future work includes analyzing modern architectures such as transformers.

\bibliography{ref}

\newpage
\appendix
\onecolumn

\section{Additional Related Work}\label{appendix: related work}

\paragraph{Feature learning in neural networks.}
Feature learning in neural networks explores how neural networks learn data patterns during training and has provided insights into phenomena such as momentum \citep{jelassi2022towards}, benign overfitting \citep{cao2022benign,kou2023benign}, adversarial training \citep{allen2022feature,li2023clean}, data augmentation \cite{zou2023benefits}, and learning on long-tailed \cite{xu2025rethinking}.
However, these approaches are not applicable to \ac{dpsgd} trained neural networks due to the noise perturbation introduced by \ac{dpsgd}.
In more detail, the key properties for feature growth may not hold.
In this paper, we develop new techniques to study the generalization performance of \ac{dpsgd} trained models.

\paragraph{Side effects in differentially private learning.}
Side effects of \ac{dp} have been widely studied in deep learning literature.
\textit{Disparate impact} was initially observed in classification tasks \cite{bagdasaryan2019differential} and generative tasks \citep{ganev2022robin}.
Then, researchers have proposed several methods, such as a regularization approach \citep{tran2021differentially}, re-weighting, and stratification methods \citep{esipova2022disparate,rosenblatt2024simple} to mitigate the disparate impacts.
A recent paper \citep{berrada2023unlocking} showed that the \ac{dp} models pre-trained with large datasets and fine-tuned with large batch size can have marginal disparate effects. 

The interplay between \ac{dp} and \textit{fairness} is also a widely studied topic. 
\cite{cummings2019compatibility} showed that {exact} fairness is not compatible with \ac{dp} under the PAC learning setting.
\cite{sanyal2022unfair} showed that it is not possible to build accurate learning algorithms that are both private and fair when data follows a {specific} kind of long-tailed distribution.
\cite{mangold2023differential} bounded the difference in fairness levels between private and non-private models under the assumption that the confidence margin is Lipschitz-continuous.
A growing body of literature highlights the inherent tension between privacy and fairness in \ac{dpsgd}, specifically exploring novel clipping modifications to restore model equity \citep{demelius2025private,soleymani2025softadaclip,zhao2025mitigating}.

Some other side effects have also been studied. 
\cite{tramer2020differentially} showed that \ac{dpl} may perform worse after \textit{bad feature learning} compared with learning handcraft features.
\cite{tursynbek2020robustness} studied \textit{adversarial robustness} in \ac{dpl} and showed that models trained by \ac{dpl} may be more vulnerable compared with non-private models.
\cite{zhang2022differentially} studied the adversarial robustness of private linear classifiers and showed differentially privately fine-tuned pre-trained models may be robust under certain parameter settings.
In addition, \cite{zhang2022closer} studied calibration of \ac{dpl} and observed \textit{miscalibration} across a wide range of vision and language tasks.
\cite{bu2023convergence} studied \ac{dpl} with \ac{ntk} and demonstrated that a large clipping threshold may benefit the calibration of \ac{dpl}.

In these works, researchers have studied the convergence of \ac{dpl} and explored explanations for the aforementioned side effects.
However, these analyses relied on some restricted assumptions that are not applicable to neural networks because (1) differentially private neural networks training is not in the \ac{ntk} regime as the noise keeps the network parameter far away from the initialization during training;
(2) training loss of ReLU neural networks is non-convex and non-smooth, contradicting the assumptions in most analyses.
In this work, we aim to overcome these {challenges} and explain the side effects of \ac{dpsgd} on a two-layer ReLU \ac{cnn}.

\section{Details about pre-training and fine-tuning data distributions}\label{appendix: pt-ft}
To illustrate the impact of pre-training, we simplify the data distribution as follows.
We consider the following pre-training and fine-tuning data distributions. 
We control their difference by a parameter $\theta$.
\paragraph{Pre-training data distribution.}
We consider a $2$-class classification problem over $2$-patch inputs. Each labelled data is denoted as $(\mathbf{x},y)$, with label $y\in \{1, 2\}$ and data vector $\mathbf{x} = \left(\mathbf{x}^{(1)}, \mathbf{x}^{(2)}\right) \in \mathbb{R}^{d\times2}$. 
A sample $(\mathbf{x},y)$ is generated from a data distribution $\mathcal{D}$ as follows.
\begin{enumerate}
	\item The label $y$ is randomly sampled from $\{1, 2\}$.
	With probability $1/2$, the label is selected as $y=1$; otherwise, it is selected as $y=2$.
	\item Each input data patch $\mathbf{x}^{(1)}, \mathbf{x}^{(2)} \in \mathbb{R}^d$ contains either feature or noise.
	\begin{itemize}
		\item Feature patch: One data patch ($\mathbf{x}^{(1)}$ or $\mathbf{x}^{(2)}$) is randomly selected as the feature patch. 
		This patch contains a feature $\mathbf{u}_{y}$ for $y\in\{1,2\}$.
		
		\item Noisy patch: The remaining patch $\boldsymbol{\xi}$ is generated from a Gaussian distribution $\mathcal{N}(0, \sigma_p^2 \mathbf{H})$, where $\mathbf{H} = \mathbf{I} - \sum_{i=1}^{2}\mathbf{u}_{i}\mathbf{u}_{i}^\top \cdot\left\|\mathbf{u}_{i}\right\|_2^{-2}$.
	\end{itemize}
\end{enumerate}
\paragraph{Fine-tuning data distribution.}
We consider a $2$-class classification problem over $2$-patch inputs. Each labelled data is denoted as $(\mathbf{x},y)$, with label $y\in \{1, 2\}$ and data vector $\mathbf{x} = \left(\mathbf{x}^{(1)}, \mathbf{x}^{(2)}\right) \in \mathbb{R}^{d\times2}$. 
A sample $(\mathbf{x},y)$ is generated from a data distribution $\mathcal{D}$ as follows.
\begin{enumerate}
	\item The label $y$ is randomly sampled from $\{1, 2\}$.
	With probability $1/2$, the label is selected as $y=1$; otherwise, it is selected as $y=2$.
	\item Each input data patch $\mathbf{x}^{(1)}, \mathbf{x}^{(2)} \in \mathbb{R}^d$ contains either feature or noise.
	\begin{itemize}
		\item Feature patch: One data patch ($\mathbf{x}^{(1)}$ or $\mathbf{x}^{(2)}$) is randomly selected as the feature patch. 
		For $y=1$, this patch contains a feature $\mathbf{u}'_{1} = \cos\theta\mathbf{u}_{1} + \sin\theta\mathbf{u}_{2}$.
		For $y=2$, this 
		this patch contains a feature $\mathbf{u}'_{2} = \cos\theta\mathbf{u}_{2} - \sin\theta\mathbf{u}_{1}$.
		
		\item Noisy patch: The remaining patch $\boldsymbol{\xi}$ is generated from a Gaussian distribution $\mathcal{N}(0, \sigma_p^2 \mathbf{H})$, where $\mathbf{H} = \mathbf{I} - \sum_{i=1}^{2}\mathbf{u}'_{i}(\mathbf{u}_{i}')^\top \cdot\left\|\mathbf{u}'_{i}\right\|_2^{-2}$.
	\end{itemize}
\end{enumerate}

\section{DP-SGD Algorithm}\label{appendix: dpsgd}

\begin{algorithm}
	\label{algo: DP-SGD}
	\caption{DP-SGD}
	\begin{algorithmic}
		\STATE {\bfseries Input:} training set $\mathcal{S}$, learning rate $\eta$, \ac{dp} noise standard deviation $\sigma_n$, batch size $B$
		%	\KwOutput{$\mathbf{w}^{(T)}$}
		%		
		\STATE initialize $\mathbf{W}^{(0)}$ randomly
		\FOR{$\text{each round}\,\,\, t = 1, 2,\cdots, T$}
		\STATE	Take a random training subset $\mathcal{S}^{(t)}$ uniformly from $\mathcal{S}$ with probability $\frac{B}{n}$		
		\STATE	Compute $\mathbf{g}^{(t)} = \frac{1}{B}\sum_{(\mathbf{x},y)\in\mathcal{S}^{(t)}}\nabla\mathcal{L}\left(\mathbf{W}^{(t-1)},\mathbf{x},y\right) + \mathcal{N}\left(0,\sigma_n^2\mathbf{I}\right)$		
		\STATE	$\mathbf{W}^{(t)} = \mathbf{W}^{(t-1)} -\eta \mathbf{g}^{(t)}$
		\ENDFOR
	\end{algorithmic}
\end{algorithm}

\section{Details about experiments}
In this section, we introduce the experimental details. We run 3 random seeds for the experiments and record the mean of them.
\subsection{Visualization of padding images}
In Figure \ref{fig: imagepadding}, we present examples of padding images.
\begin{figure}
	\centering
	\includegraphics[width=0.8\linewidth]{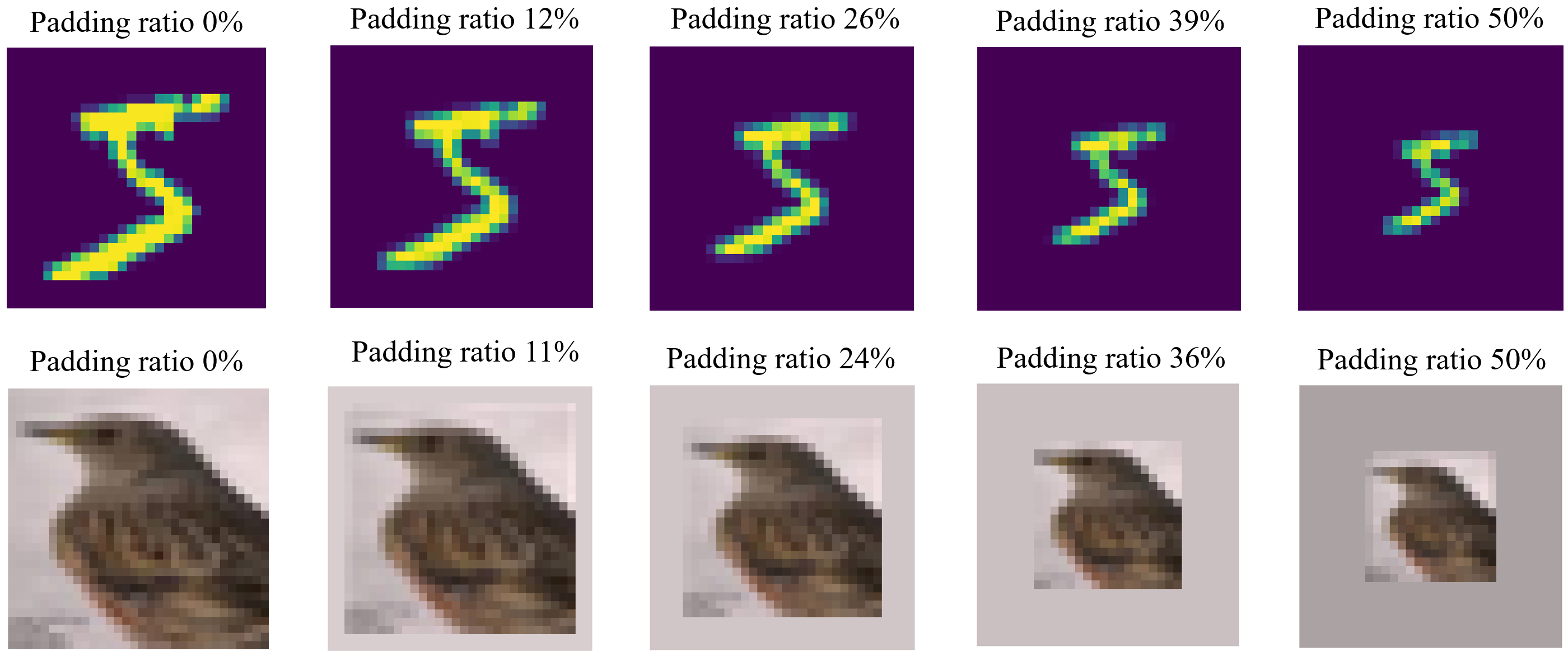}
	\caption{Examples on using image padding to control feature sizes. For digits and objects, we pad the images with their background color.}
	\label{fig: imagepadding}
\end{figure}

\subsection{Experimental details of Figure \ref{fig:phase}}\label{app: figure1}
Because we mainly characterize the privacy-utility tradeoff, we simply set 2 classes with equal feature sizes and equal dataset sizes $100$.
We vary the feature size from $0$ to $21$ and the \ac{dpsgd} noise standard deviation $\sigma_n$ from $0$ to $4$.
We set the noise patch standard deviation $\sigma_p$ as $0.02$.
We set the clipping threshold as $C=2$.
We set the number of neurons as $m=32$.

\section{Additional experiments}
\subsection{Validation of the impact of class sizes}\label{app: impact_size}
We conducted additional experiments to evaluate performance under varying degrees of class imbalance. Specifically, we employed the MNIST dataset, fixing the population of all classes except Class 5 to 3,000 samples. We then systematically varied the sample size of Class 5. All models were trained with a fixed privacy budget of $(3,10^{-5})$-DP
. The impact of these variations on the test accuracy for Class 5 is summarized in the table below.
\begin{table}[h]
\centering
\begin{tabular}{lcccc}
\toprule
\textbf{Class 5 Size} & \textbf{100} & \textbf{500} & \textbf{1000} & \textbf{3000} \\
\hline
\textbf{Accuracy (\%)} & 71 & 93 & 96 & 97 \\
\hline
\textbf{Accuracy (Adv) (\%)} & 60 & 83 & 96 & 97 \\
\bottomrule
\end{tabular}
\end{table}
These results corroborate our theoretical findings, showing that test accuracy degrades with reduced sample sizes.

\subsection{Experiments of vision transformers on padded images}\label{app: vit}
we conducted additional experiments using Vision Transformers (ViTs) on padded MNIST datasets. The results, presented in the table below, demonstrate that the feature-to-noise ratio remains a critical metric for characterizing the generalization performance of DP-SGD, even when applied to Transformer-based architectures.
\begin{table}[h]
\centering
\begin{tabular}{lcccc}
\toprule
\textbf{Padding ratio} & \textbf{0\%} & \textbf{26\%} & \textbf{50\%} & \textbf{75\%} \\
\midrule
MNIST & 93\% & 93\% & 93\% & 91\% \\
\midrule
MNIST (adv) & 78\% & 77\% & 66\% & 20\% \\
\bottomrule
\end{tabular}
\end{table}

\section{Some datasets and neural network architectures observing side effects}\label{app: data_architecture}
\cite{bagdasaryan2019differential} has identified disparate effects in datasets such as MNIST, CelebA, and Twitter posts using ResNet and LSTM networks. 
Unfairness has been observed in CelebA and CIFAR-10 datasets using ResNet networks \citep{sanyal2022unfair}.
\cite{zhang2022closer} observed miscalibration in QNLI, QQP, SST-2 with fine-tuned RoBERTa-base.

\section{Experiments about network freezing}    \label{appendix: freezing}
In this section, we conduct experiments to verify the effectiveness of network freezing. 
Specifically, we consider the following stagewise network freezing algorithm.

\begin{algorithm}
	\label{algo: freezing}
	\caption{Stage-wise Network Freezing}
	\begin{algorithmic}
		\STATE {\bfseries Input:} training set $\mathcal{S}$, freezing iterations $\mathcal{T}_f$, pruning percentage $v$
		%	\KwOutput{$\mathbf{w}^{(T)}$}
		%		
		\STATE initialize $\mathbf{W}^{(0)}$ randomly
            \STATE Identify the set of all prunable structures in the model, $\mathcal{Q} = \{q_1, q_2, \dots, q_L\}$.
		\FOR{$\text{each round}\,\,\, t = 1, 2,\cdots, T$}
        	\IF{$t\in T_f$} 
                    
                    \STATE For each structure $q_i \in \mathcal{Q}$, calculate its importance score, $m_i$.
                    \STATE Calculate the number of structures to prune: $k \leftarrow \lfloor \frac{v}{100} \times m \rfloor$.

		\STATE Freeze the set $\mathcal{Q}_\text{prune} \subset \mathcal{Q}$ containing the $k$ structures with the lowest importance scores.
                \ENDIF
		\STATE	$\mathbf{W}^{t}\leftarrow$PrivateTraining($\mathbf{W}^{(t-1)},\mathcal{S}$)	
		\ENDFOR
	\end{algorithmic}
\end{algorithm}

In the following experiments, we consider magnitude-based unstructured pruning. 
The parameters are set as follows:
\begin{itemize}
    \item DP parameters: $(1,10^{-5})$-DP
    \item Total number of epochs: $T= 10$
    \item Pruning stages $\mathcal{T}_f = \{1,2,3\}$
    \item Percentage of pruning: 77\%
\end{itemize}

The accuracy for LeNet on MNIST is 
\begin{table}[h]
	%\caption{Sample table title}
	\label{table: freezing}
	\begin{center}
		\begin{tabular}{lcc}
        \toprule
			&With freezing  & Without freezing\\
			\midrule
			Accuracy&96.71\% & 95.11\% \\
            \bottomrule
		\end{tabular}
	\end{center}
\end{table}

\section{Discussion about data augmentations}   \label{appendix: data augmentation}
Data augmentation is verified to be useful in Figure 3 in [12]. 
It can increase 2\% accuracy on CIFAR-10 image classification.

\section{Overview of Challenges and Proof Sketch}\label{app: sketch}
In this section, we outline the main challenges in studying feature learning of \ac{dpsgd} on \acp{cnn} and the key proof techniques employed to overcome the challenges.

\subsection{Challenge 1: Non-Smoothness of the ReLU  Activation Function}
The first challenge arises from the non-smoothness of the ReLU activation function.
Some existing papers (e.g., \citep{girgis2021shuffled}) analyzed \ac{dpsgd} with Lipschitz-smoothness-based approaches.
However, this kind of approach is not applicable to ReLU neural networks. 

In a two-layer \acp{cnn}, we can track the neurons' feature learning process through their gradients.
For any $i \in \{1,2\}, r \in [m]$, the gradient on $\mathbf{w}_{i,r}^{(t)}$ can be decomposed as
\begin{equation}\nonumber    \label{eq: gradient}
	\begin{aligned}
		\nabla_{\mathbf{w}_{i,r}^{(t)}} \mathcal{L}_{\mathcal{S}}(\mathbf{W}^{(t)})=&\underbrace{\sum_{j\in\{\text{maj},\text{min}\}}(\mu_{i,j} \mathbf{u}_{i,j}- \mu_{3-i,j} \mathbf{u}_{3-i,j})}_\text{Data features}
		+ \underbrace{\sum_{k=1}^{n}\boldsymbol{\xi}_k (\bar{\rho}_{i,j,k}\mathbb{I}(y_k=i)-\underline{\rho}_{i,j,k}\mathbb{I}(y_k\neq i))}_\text{Data noise},
	\end{aligned} 
\end{equation}
where $\mu_{i,j},\mu_{3-i,j},\bar{\rho}_{i,j,k},\underline{\rho}_{i,j,k} \ge0$ are constants. The neurons tend to learn both class-relevant data features and data noise of the targeted class while unlearning others. 

Some existing approaches (e.g., \citep{cao2022benign}) bound the feature learning process by characterizing the leading neurons that learn the most features.
However, this approach only works for ReLU$^q$ ($q>2$) activation functions, where the leading neurons dominate other neurons during training. As the ReLU function is piece-wise linear, these approaches fail in ReLU neural networks.    

To overcome this challenge, we study the feature learning process by analyzing the dynamics of the model outputs $F_i(\mathbf{W},\mathbf{x}), i\in[2]$ defined in Section \ref{sec: model}.

\subsection{Challenge 2: Randomness from DP-SGD}
The second challenge stems from the randomness introduced by \ac{dpsgd}.
The learning process is significantly perturbed due to the random noise in \ac{dpsgd}.
\cite{kou2023benign} attempted to bound the feature learning process based on the monotonicity of the weights of feature vectors.
However, due to the randomness from \ac{dpsgd}, the weights are not consistently increasing.

To address this challenge, we track the increments of the model outputs for any data point $(\mathbf{x},y) \sim \mathcal{D}_{i,j}$ instead.
The key proposition is presented as follows.

\begin{proposition}\label{proposition: F_increment}
	For any $(\mathbf{x},y)\sim\mathcal{D}_{i,j}, i\in [2], j\in\{\text{maj},\text{min}\}$, with probability at least $1-\delta$,
	\begin{itemize}
		\item The increment of model output for the targeted class $y$ satisfies
		\begin{equation}\label{equ: target output increment}
			\begin{aligned}
				\Delta^{(t)}_y(\mathbf{x}) =&F_y\left(\mathbf{W}^{(t+1)},\mathbf{x}\right) - F_y\left(\mathbf{W}^{(t)},\mathbf{x}\right)\\
				\ge& \Omega\left(\frac{\eta}{\sqrt{m}}\cdot\gamma_{i,j}\cdot\Lambda_{i,j}\cdot\mathbb{E}_{\left(\mathbf{x},y\right)\sim \mathcal{D}_{i,j}}\!\left[1-\textnormal{prob}_{y}\left(\mathbf{W}^{(t)},\mathbf{x}\right)\right]\right)\cdot\left\|\mathbf{u}_{i,j}\right\|_2^2\\
				&- \tilde{\mathcal{O}}\left(\frac{\eta}{m}\left\|\mathbf{u}_{i,j}\right\|_2^2+ \eta\sigma_n\left\|\mathbf{u}_{i,j}\right\|_2 +\frac{\eta}{m\sqrt{n}}\sqrt{d}\sigma_p^2 +\eta \sqrt{d}\sigma_n\sigma_p\right).
			\end{aligned}
		\end{equation}
		\item The increment of model output for the other class $3-y$ satisfies
		\begin{equation}
			\begin{aligned}
			\Delta^{(t)}_{3-y}(\mathbf{x}) = &F_{3-y}\left(\mathbf{W}^{(t+1)},\mathbf{x}\right) - F_{3-y}\left(\mathbf{W}^{(t)},\mathbf{x}\right)\\
				\le&\tilde{\mathcal{O}}\left(\eta\sigma_n\left\|\mathbf{u}_{i,j}\right\|_2+\frac{\eta}{m\sqrt{n}}\sqrt{d}\sigma_p^2 +\eta \sqrt{d}\sigma_n\sigma_p\right).\\
			\end{aligned}
		\end{equation}
	\end{itemize}
\end{proposition}

Proposition \ref{proposition: F_increment} shows that the model output on the data $(\mathbf{x},y)$ with respect to the targeted class, $F_y(\mathbf{W},\mathbf{x})$, tends to increase over iterations (see term 1 in the RHS of (\ref{equ: target output increment})).
However, due to the noise perturbation introduced by \ac{dpsgd}, the model increment $\Delta^{(t)}_{y}$ becomes smaller and cannot always be positive (because of the second negative term in RHS of (\ref{equ: target output increment})).
In addition, the model output on the data $(\mathbf{x},y)$ with respect to the other class may increase due to the randomness from batches and \ac{dpsgd}.
The results of Proposition \ref{proposition: loss increment} allow us to track the test loss increments, as shown in the following subsection.

\subsection{Challenge 3: Non-Linearity of Cross-Entropy and Softmax Functions}
Due to the non-linearity of cross-entropy and softmax functions, the model output increment bounds in Proposition \ref{proposition: F_increment} cannot be directly applied to bounding test loss.

To tackle this challenge, we bound the non-linear functions with a piece-wise linear function, as stated in Lemma \ref{lemma: non-linear bound}.

\begin{lemma}\label{lemma: non-linear bound}
	Under Assumption \ref{assumption: non-perfect}, we have
	\begin{equation}\nonumber
		\begin{aligned}
			&\mathcal{L}\left(\mathbf{W}^{(t+1)},\mathbf{x},y\right) \!-\! \mathcal{L}\left(\mathbf{W}^{(t)},\mathbf{x},y\right)
			\!\le\! c_1\cdot\sigma\!\left(\Delta^{(t)}_{3-y}(\mathbf{x}) -\Delta^{(t)}_y(\mathbf{x})\right) \!-\! c_2\cdot\sigma\!\left(\Delta^{(t)}_y(\mathbf{x})-\Delta^{(t)}_{3-y}(\mathbf{x})\right)\!,
		\end{aligned}
	\end{equation}
	for some constants $c_1,c_2 >0$.
\end{lemma}	

Lemma \ref{lemma: non-linear bound} allows us to apply the model output increment bounds in Proposition \ref{proposition: F_increment} to bound the test loss, as shown in Proposition \ref{proposition: loss increment}.

\begin{proposition}\label{proposition: loss increment}
	Under Condition \ref{condition} and Assumption \ref{assumption: non-perfect}, with probability at least $1-\delta$, for any $i\in[2], j\in\{\text{maj}, \text{min}\}$, we have
	\begin{equation}\nonumber
		\begin{aligned}
			&\mathcal{L}_{\mathcal{D}_{i,j}}\!\left(\!\mathbf{W}^{\left(t+1\right)}\!\right) \!-\! \mathcal{L}_{\mathcal{D}_{i,j}}\!\left(\!\mathbf{W}^{(t)}\!\right) \!\le\!-\Omega\!\left(\!\frac{\eta}{\sqrt{m}}\cdot\gamma_{i,j}\cdot\Lambda_{i,j}\cdot\left\|\mathbf{u}_{i,j}\right\|_2^2\right)\!\cdot\mathbb{E}_{(\mathbf{x},y)\sim \mathcal{D}_{i,j}}\!\left[1\!-\!\textnormal{prob}_{y}\!\!\left(\!\mathbf{W}^{(t)},\mathbf{x}\!\right)\right]\\
			&+\underbrace{\tilde{\mathcal{O}}\left(\frac{\eta}{m}\sqrt{\frac{1}{n}}\left\|\mathbf{u}_{i,j}\right\|_2^2+\eta\sigma_n\left\|\mathbf{u}_{i,j}\right\|_2+\frac{\eta}{m\sqrt{n}}\sqrt{d}\sigma_p^2 +\eta \sqrt{d}\sigma_n\sigma_p\right)}_{:=\phi}.
		\end{aligned}
	\end{equation}
\end{proposition}
With the fact that under Assumption \ref{assumption: non-perfect},
\begin{align}
	1-\text{prob}_{y}\left(\mathbf{W}^{(t)},\mathbf{x}\right) = \Theta\left(1\right)\cdot\mathcal{L}\left(\left(\mathbf{W}^{(t)}\right),\mathbf{x},y\right)
\end{align}
holds, Proposition \ref{proposition: loss increment} can be applied to establish the following test loss bound
\begin{equation}\label{equ: loss_re}
	\begin{aligned}
		\mathcal{L}_{\mathcal{D}_{i,j}}\left(\mathbf{W}^{\left(t+1\right)}\right) \!\le&\! \left(1\!-\!\Omega\left(\frac{\eta}{\sqrt{m}}\cdot\gamma_{i,j}\cdot \Lambda_{i,j}\cdot\left\|\mathbf{u}_{i,j}\right\|_2^2\right)\!\right)\!\cdot\mathcal{L}_{\mathcal{D}_{i,j}}\!\left(\mathbf{W}^{\left(t\right)}\right) + \phi.
	\end{aligned}
\end{equation}
Recursively applying (\ref{equ: loss_re}) over $T$ iterations yields Theorem \ref{theorem: convergence}.

\section{Proof}\label{app:proof}
\subsection{Preliminaries}
\begin{lemma}[Gaussian distribution tail bound]\label{lemma: G tail}
	A variable $x$ following $\mathcal{N}(0, \sigma_0^2)$ satisfies
	\begin{align}
		\mathbb{P}[x \ge t\sigma_0], \mathbb{P}[x \le -t\sigma_0] \le \exp\left(-\frac{t^2}{2}\right),\forall t\ge0.
	\end{align}
\end{lemma}

\begin{lemma}[Chi-squared distribution tail bound]\label{lemma: chi-squared}
	A variable $\mathbf{x} \sim \mathcal{N}(0, \sigma_0^2\mathbf{I}_d)$ satisfies
	With probability at least $1-\exp(-td/10)$, we have
	\begin{align}
		\mathbb{P}\left[\left\|\mathbf{x}\right\|_2\ge \sigma_0\sqrt{2(t+1)d}\right] \le \exp(-(t+1)d/10), \forall t \ge 0.
	\end{align}
\end{lemma}

\begin{lemma}\label{lemma: norm}
    Suppose that $\delta >0$ and $d = \Omega(\log(n/\delta))$. With probability $1-\delta$, we have
    \begin{align}
       \frac{\sigma_p\sqrt{d}}{2} \le\|\boldsymbol{\xi}_i\|_2^2 \le \frac{3\sigma_p\sqrt{d}}{2}.
    \end{align}
\end{lemma}
\begin{proof}
    By Bernstein’s inequality, with probability at least $1-\delta/n$, we have
    \begin{align}
        |\left\|\boldsymbol{\xi}_i\right\|_2^2-\sigma_p(d-4)| = \mathcal{O}(\sigma_p^2\sqrt{d\log(2n/\delta)}). 
    \end{align}
    Applying a union bound finishes the proof.
\end{proof}

\begin{lemma}[Half-normal distribution concentration bound]\label{lemma: half normal}
	Suppose $x_1,x_2,\cdots,x_n \sim \mathcal{N}(0,\sigma_0^2)$.
	Then, with probability at least $1-\delta$,
    \begin{align}
       \sqrt{\frac{2}{\pi}}\sigma_0 - \frac{\sqrt{2\log(2/\delta)}}{\sqrt{n}}\sigma_0\le \frac{1}{n}\sum_{i=1}^{n}\left|x_i\right| \le \sqrt{\frac{2}{\pi}}\sigma_0 + \frac{\sqrt{2\log(2/\delta)}}{\sqrt{n}}\sigma_0.
    \end{align}
\end{lemma}
\begin{proof}
	First, half-normal variables $|x_i|,\forall i\in[n]$ are sub-Gaussian as a half-normal variable has a negative tail bounded by $-\sqrt{\frac{2}{\pi}}$ and a Gaussian delay positive tail.
	Then, by Hoeffding's inequality, we have
	\begin{align}
		\mathbb{P}\left[\frac{1}{n}\sum_{i=1}^{n}\left|x_i\right| - \sqrt{\frac{2}{\pi}}\sigma_0 \ge t\sigma_0\right], \mathbb{P}\left[\frac{1}{n}\sum_{i=1}^{n}\left|x_i\right| - \sqrt{\frac{2}{\pi}}\sigma_0 \le -t\sigma_0\right]\le \exp\left(-\frac{nt^2}{2}\right), \forall t\ge 0.
	\end{align}
    Therefore, with probability $1-\delta$, we have
    \begin{align}
       \sqrt{\frac{2}{\pi}}\sigma_0 - \frac{\sqrt{2\log(2/\delta)}}{\sqrt{n}}\sigma_0\le \frac{1}{n}\sum_{i=1}^{n}\left|x_i\right| \le \sqrt{\frac{2}{\pi}}\sigma_0 + \frac{\sqrt{2\log(2/\delta)}}{\sqrt{n}}\sigma_0.
    \end{align}
    This completes the proof.
\end{proof}
\begin{lemma}\label{lemma: num concentration}
	Let $x_1, \cdots, x_m$ be $m$ independent zero-mean Gaussian variables.
	Denote $z_i$ as indicators for signs of $x_i$, i.e., for all $i \in [m]$,
	\begin{align}
		z_i = 
		\begin{cases}
			1, & x_i>0,\\
			0, & x_i\le0.
		\end{cases}
	\end{align}
	Then, we have
	\begin{align}
		\mathbb{P}\left[\sum_{i=1}^{m}z_i \ge \frac{m}{4}\right] \ge 1-\exp\left(-2m\right).
	\end{align}
\end{lemma}

\begin{proof}
	Because $z_i, i\in[m]$ are bounded in $[0,1]$, $z_i, i\in[m]$ are sub-Gaussian variables.
	By Hoeffding’s inequality, we have
	\begin{align}
		\mathbb{P}\left[m\cdot\left(\frac{1}{m}\sum_{i=1}^{m}z_i\right)\le m\cdot\left(\frac{1}{2}-\epsilon\right)\right] \le \exp\left(\frac{2m^2\epsilon^2}{m(\frac{1}{16})}\right).
	\end{align}
	Let $\epsilon = \frac{1}{4}$, we have
	\begin{align}
		\mathbb{P}\left[\sum_{i=1}^{m}z_i \le \frac{m}{4}\right] \le \exp\left(-2m\right).
	\end{align}
	Therefore, we have
	\begin{align}
		\mathbb{P}\left[\sum_{i=1}^{m}z_i \ge \frac{m}{4}\right] \ge 1 - \exp\left(-2m\right).
	\end{align}
	This completes the proof.
\end{proof}

\begin{lemma}\label{lemma: two_gaussian}
	Let $x_1$ be a Gaussian variable following $\mathcal{N}(0,\sigma_1)$ and $x_2$ be a Gaussian variable following $\mathcal{N}(0,\sigma_2)$.
	Then, with probability at least $1-\exp(\frac{t_1^2}{2})-\exp(\frac{t_2^2}{2})$, we have
	\begin{align}
		\langle x_1,x_2 \rangle \le t_1t_2\sigma_1\sigma_2.
	\end{align}
\end{lemma}
Lemma \ref{lemma: two_gaussian} can be proved by using Lemma \ref{lemma: G tail}.

\begin{lemma}\label{lemma: sgd}
	For $N$ \ac{iid} random variables $x_1, \cdots, x_N \in[0,1]$ with expectation $\mu$, with probability $1-\delta$, we have
	\begin{align}
		\frac{1}{N}\sum_{i=1}^{N} x_i - \mu \le \sqrt{\frac{\log(1/\delta)}{2N}}.
	\end{align}
	with $t>0$.
\end{lemma}
\begin{proof}
    From Hoeffding's inequality, we have 
	\begin{align}
		\mathbb{P}\left[\frac{1}{N}\sum_{i=1}^{N} x_i - \mu \le \sqrt{\frac{t}{N}}\right] \ge 1 - \exp\left(-2t\right)
	\end{align}
    Letting the right hand side be $\delta$, we have that with probability $1-\delta$,
    \begin{align}
        \frac{1}{N}\sum_{i=1}^{N} x_i - \mu \le \sqrt{\frac{\log(1/\delta)}{2N}}.
    \end{align}
    This completes the proof.
\end{proof}

\begin{lemma}\label{lemma: Gamma}
	For any constant $t\in (0,1]$ and $x \in [-a,b], a,b >0$, we have 
	\begin{align}
		\log(1+t\cdot(\exp(x)-1)) \le \Gamma(x) x,
	\end{align}
	where $\Gamma(x) = \mathbb{I}(x\ge0) + \left[\frac{\log(1+t\cdot(\exp(-a)-1))}{-a}\right]\cdot\mathbb{I}(x<0)$.
\end{lemma}

\begin{proof}
	First, considering $x\ge0$, we have
	\begin{align}
		\frac{\partial\log(1+t\cdot(\exp(x)-1))}{\partial t} =& \frac{\exp(x)-1}{1+t\cdot(\exp(x)-1)}\ge0.
	\end{align}
	Thus, $\log\left(1+t\cdot(\exp(x)-1)\right) \le x, \forall x>0$.
	Second, considering $x<0$, we have
	\begin{align}
		\frac{\partial^2 \log(1+tv(\exp(x)-1))}{\partial x^2} = \frac{(1-t)t\exp(x)}{[1+t(\exp(x)-1)]^2}\ge 0.
	\end{align}
	So $\log(1+t(\exp(x)-1))$ is a convex function of $x$. 
	We can conclude that
	\begin{align}
		\log(1+t(\exp(x)-1)) \le \frac{\log(1+t(\exp(-a)-1))}{-a} x, \forall x<0.
	\end{align}
	This completes the proof.
\end{proof}
\begin{lemma}\label{lemma: -logx and 1-x}
	For $x\in [x_0, 1]$ and $x_0>0$, we have
	\begin{align}
		1-x \ge \frac{1-x_0}{-\log(x_0)}\cdot\left(-\log(x)\right).
	\end{align}
\end{lemma}
Lemma \ref{lemma: -logx and 1-x} can be proved by applying the convexity of $-\log(x)$.

\begin{lemma}\label{lemma: geometric}
	For a geometric sequence defined as $z_{t+1} = \beta z_t$ for a constant $\beta < 1$, we have
	\begin{align}
		\sum_{t=1}^{T} z_t = \frac{1-\beta^T}{1-\beta}\cdot z_1.
	\end{align}
\end{lemma}

Lemma \ref{lemma: geometric} is obtained from the property of Geometric sequences.

\begin{lemma}\label{lemma: exp}
	For $x\le x_0$, we have
	\begin{align}
		\frac{1}{\exp(x)+1} \ge \frac{1}{\exp(\bar{x}_0)+1} -\frac{\exp(\bar{x}_0)}{(\exp(\bar{x}_0)+1)^2}\cdot(x-\bar{x}_0),
	\end{align}
	where $\bar{x}_0 = |x_0|$.
\end{lemma}
Lemma \ref{lemma: exp} can be prove by the monotonicity and convexity of function $f(x) = \frac{1}{\exp(x)+1}$ with $x>0$.
\begin{lemma}\label{lemma: gradient norm}
	In each iteration $t$, with probability at least $1-\delta$, for any $(\mathbf{x},y) \in \mathcal{D}_{i,j}$, we have
	\begin{align}
		\left\|\nabla_{\mathbf{W}^{(t)}} \mathcal{L}(\mathbf{W}^{(t)},\mathbf{x},y)\right\|_2 \le \mathcal{O}\left(\left\|\mathbf{u}_{i,j}\right\|_2+\sigma_p\sqrt{d}\right).
	\end{align}
\end{lemma}
Lemma \ref{lemma: gradient norm} follows from Lemma \ref{lemma: norm}.

\begin{lemma}\label{lemma: log lower bound}
	For a variable $x\in[a,b] (a<0, b>0)$, the function $f(x) = \log(1+x)$ satisfies
	\begin{align}
		f(x)\ge \frac{\log(1+b)}{b}x \cdot \mathbb{I}(x\ge0) + \frac{\log(1+a)}{-a}x\cdot\mathbb{I}(x<0).
	\end{align}
\end{lemma}
Lemma \ref{lemma: log lower bound} can be proved by the monotonicity and concavity of the $\log(\cdot)$ function.
\begin{lemma}\label{lemma: squared_norm}
    For any $(\mathbf{x},y)\sim\mathcal{D}$, With probability at least $1-1/d$, we have
    \begin{align}
      \frac{\sigma_p^2d}{2}  \le \left\|\boldsymbol{\xi}\right\|_2 \le \frac{3\sigma_p^2d}{2}.
    \end{align}
\end{lemma}
\begin{proof}
    By Bernstein's inequality, with probability $1-1/d$, we have
    \begin{align}
        |\left\|\boldsymbol{\xi}\right\|_2-\sigma_p^2(d-2)| = \mathcal{O}(\sigma_p^2\sqrt{d\log(2d)}).
    \end{align}
    As $d\ge 50$, we have
    \begin{align}
        \frac{\sigma_p^2d}{2}\le\left\|\boldsymbol{\xi}\right\|_2\le\frac{3\sigma_p^2d}{2},
    \end{align}
    with probability $1-1/d$.
\end{proof}
\begin{lemma}\label{lemma: inner_prod}
    For any $(\mathbf{x}_1,y_1),(\mathbf{x}_2,y_2) \sim\mathcal{D}$, with probability at least $1-\delta$, we have
    \begin{align}
      |\langle\boldsymbol{\xi}_1, \boldsymbol{\xi}_2\rangle| \le 2\sigma_p^2\sqrt{d\log(2/\delta)}.
    \end{align}
\end{lemma}
This conclusion holds by Bernstein's inequality. 

\subsection{Proof of Theorem \ref{theorem: convergence}}\label{appendix: theorem convergence proof}
In this subsection, we will prove Theorem \ref{theorem: convergence}. 
For convenience, we first define the clipping multiplier of data $(\mathbf{x},y)$ as
\begin{align}\label{equ: clipping factor}
	h(C,\mathbf{x},y) = \frac{1}{\max\left\{1,\frac{\left\|\nabla\mathcal{L}\left(\mathbf{W}^{(t)},\mathbf{x},y\right)\right\|_2}{C}\right\}}.
\end{align}

Then, we compute the gradient of the neural networks and prove a bound for it.
\subsubsection{Network Gradient}
The stochastic gradient on $\mathbf{w}_{q,r}, q\in\{1,2\}$ at iteration $t$ is 
\begin{equation}
	\begin{aligned}
		&\nabla_{\mathbf{w}_{q,r}^{(t)}} \mathcal{L}_{\mathcal{S}}(\mathbf{W}^{(t)})\\
		=& -\frac{1}{mB}\cdot\sum_{(\mathbf{x},y)\in \mathcal{S}^{(t)}}\left[\mathbb{I}\left(y=q\right)\cdot\left(1-\text{prob}_q(\mathbf{W}^{(t)},\mathbf{x})\right)\cdot\sum_{j=1}^{2}\sigma'\left(\left\langle \mathbf{w}_{q,r}^{(t)}, \mathbf{x}^{(j)}\right\rangle\right)\cdot\mathbf{x}^{(j)}\right]\\
		&+ \frac{1}{mB}\cdot\sum_{(\mathbf{x},y)\in \mathcal{S}^{(t)}}\left[\mathbb{I}\left(y\neq q\right)\cdot\text{prob}_q(\mathbf{W}^{(t)},\mathbf{x})\cdot\sum_{j=1}^{2}\sigma'\left(\left\langle \mathbf{w}_{q,r}^{(t)}, \mathbf{x}^{(j)}\right\rangle\right)\cdot\mathbf{x}^{(j)}\right].
	\end{aligned}
\end{equation}
We assume that the gradient of the ReLU activation function at 0 to be $\sigma'(0) = 1$
without losing generality.

\subsubsection{Bound of the Clipping Multiplier $h(C,\mathbf{x},y)$}
By definition (\ref{equ: clipping factor}), we know that
\begin{align}
	h(C,\mathbf{x},y) \le 1.
\end{align}
In addition, from Lemma \ref{lemma: gradient norm}, we know that with probability at least $1-\delta$,
\begin{align}
	h(C,\mathbf{x},y) \ge \Omega\left(\frac{C}{\left\|\mathbf{u}_{i,j}\right\|_2+\sigma_p\sqrt{d}}\right).
\end{align}

\subsubsection{Loss Increment}
For any data $(\mathbf{x},y) \sim \mathcal{D}_{i,j}, i\in \{1,2\}, j\in\{\text{Maj},\text{Min}\}$, with some rearrangement, we can express the increment of the loss as
\begin{equation}\label{equ: loss}
	\begin{aligned}
		&\mathcal{L}(\mathbf{W}^{(t+1)},\mathbf{x},y) - \mathcal{L}(\mathbf{W}^{(t)},\mathbf{x},y)=-\log\left(\text{prob}_y\left(\mathbf{W}^{(t+1)},\mathbf{x}\right)\right) + \log\left(\text{prob}_y\left(\mathbf{W}^{(t)},\mathbf{x}\right)\right)\\
		=& \log\left(1+ \left(1-\text{prob}_y\left(\mathbf{W}^{(t)},\mathbf{x}\right)\right)\left(\exp\left(\Delta_{3-y}^{(t)}\left(\mathbf{x}\right)-\Delta_y^{(t)}\left(\mathbf{x}\right)\right)-1\right)\right),
	\end{aligned}
\end{equation}
where $\Delta_y^{(t)}\left(\mathbf{x}\right) = F_y^{(t+1)}\left(\mathbf{x}\right) - F_y^{(t)}\left(\mathbf{x}\right), \Delta_{3-y}^{(t)}\left(\mathbf{x}\right) = F_{3-y}^{(t+1)}\left(\mathbf{x}\right) - F_{3-y}^{(t)}\left(\mathbf{x}\right)$ represent the model output increments at iteration $t$.
As we can see in (\ref{equ: loss}), one key factor that control the loss increment is $\Delta_{3-y}^{(t)}\left(\mathbf{x}\right)-\Delta_y^{(t)}\left(\mathbf{x}\right)$.
We then bound the term it as follows.
We first decompose $\Delta_{3-y}^{(t)}\left(\mathbf{x}\right)-\Delta_y^{(t)}\left(\mathbf{x}\right)$ as following.
\begin{equation}\label{equ: delta}
	\begin{aligned}
		&\Delta_{3-y}^{(t)}\left(\mathbf{x}\right)-\Delta_y^{(t)}\left(\mathbf{x}\right)\\
		=& \underbrace{\frac{1}{m}\sum_{r=1}^{m}\left[\sigma\!\left(
			\!\left\langle \mathbf{w}^{(t+1)}_{3-y,r},\mathbf{u}_{i,j}\right\rangle\!\right) \!-\! \sigma\!\left(\!
			\left\langle \mathbf{w}^{(t)}_{3-y,r},\mathbf{u}_{i,j}\right\rangle\!\right)\right]}_{A_1} \!+\! \underbrace{\frac{1}{m}\sum_{r=1}^{m}\left[\sigma\!\left(\!
			\left\langle \mathbf{w}^{(t+1)}_{3-y,r},\boldsymbol{\xi}\right\rangle\! \right) \!-\! \sigma\!\left(\!
			\left\langle \mathbf{w}^{(t)}_{3-y,r},\boldsymbol{\xi}\right\rangle \!\right)\right]}_{A_2}\\
		&-\! \underbrace{\frac{1}{m}\sum_{r=1}^{m}\left[\sigma\!\left(
			\left\langle \mathbf{w}^{(t+1)}_{y,r},\mathbf{u}_{i,j}\right\rangle\right) \!-\! \sigma\!\left(
			\left\langle \mathbf{w}^{(t)}_{y,r},\mathbf{u}_{i,j}\right\rangle\right)\right]}_{A_3} \!-\! \underbrace{\frac{1}{m}\sum_{r=1}^{m}\left[\sigma\left(
			\left\langle \mathbf{w}^{(t+1)}_{y,r},\boldsymbol{\xi}\right\rangle \right) \!-\! \sigma\!\left(
			\left\langle \mathbf{w}^{(t)}_{y,r},\boldsymbol{\xi}\right\rangle \right)\right]}_{A_4},
	\end{aligned}
\end{equation}
where $\boldsymbol{\xi}$ is the noise patch of a data sample $(\mathbf{x},y)$ generated from $\mathcal{D}_{i,j}$.
We then upper bound $A_1$, $A_2$ and lower bound $A_3,A_4$ to find the upper bound of $\Delta_{3-y}^{(t)}\left(\mathbf{x}\right)-\Delta_y^{(t)}\left(\mathbf{x}\right)$.

Here, we prove that $\Delta_{3-y}^{(t)}\left(\mathbf{x}\right)-\Delta_y^{(t)}\left(\mathbf{x}\right)$ is bounded.
\begin{lemma}\label{lemma: update_bounded}
For any $(\mathbf{x},y)\sim \mathcal{D}$, with probability at least $1-\delta$, we have
    \begin{align}
        |\Delta_{3-y}^{(t)}\left(\mathbf{x}\right)-\Delta_y^{(t)}\left(\mathbf{x}\right)| \le \mathcal{O}\left(\eta (C+\sqrt{d}\sigma_n)(\max_{i,j}\left\|\mathbf{u}_{i,j}\right\|_2 + \sqrt{d}\sigma_p)\right).
    \end{align}
\end{lemma}
\begin{proof}
With Lemma \ref{lemma: chi-squared}, we have
\begin{equation}\label{equ: update_bounded}
\begin{aligned}
	|\Delta_{3-y}^{(t)}\left(\mathbf{x}\right)-\Delta_y^{(t)}\left(\mathbf{x}\right)|\le& 2\eta\left(C+ \left\|\mathbf{n}^{(t)}\right\|_2\right)\left(\max_{i,j}\left\|\mathbf{u}_{i,j}\right\|_2 + \left\|\boldsymbol{\xi}\right\|_2 \right)\\
	\le& \mathcal{O}\left(\eta (C+\sqrt{d}\sigma_n)(\max_{i,j}\left\|\mathbf{u}_{i,j}\right\|_2 + \sqrt{d}\sigma_p)\right),
\end{aligned}
\end{equation}
with probability at least $1-\delta$.
By the learning rate condition in Condition \ref{condition}, we can conclude that $|\Delta_{3-y}^{(t)}\left(\mathbf{x}\right)-\Delta_y^{(t)}\left(\mathbf{x}\right)|$ is upper bounded by a constant.
\end{proof}

For the term $A_1$, with probability at least $1-\delta/T$, we have the following inequality,
\begin{equation}
	\begin{aligned}
		A_1\! =& \frac{1}{m}\!\sum_{r=1}^{m}\!\left[\sigma\!\!\left(\!\!\left\langle\!\mathbf{w}_{3-y,r}^{(t)}\!-\!\frac{\eta}{mB}\cdot\!\!\!\!\sum_{(\mathbf{x}_k,y_k)\in\mathcal{S}^{(t)}_{i,j}}\!\!\sigma'\!\left(\!\left\langle\!\mathbf{w}_{3-y,r}^{(t)},\mathbf{u}_{i,j}\!\right\rangle\!\right)\!\cdot h(C,\mathbf{x}_k,y_k)\cdot\text{prob}_{3-y}\!\left(\!\mathbf{W}^{(t)},\mathbf{x}_k\!\right)\!\cdot\mathbf{u}_{i,j}\right.\right.\right.\\
		&+\left.\left.\left.\eta\cdot\mathbf{n}^{(t)}_{3-y,r},\mathbf{u}_{i,j}\right\rangle\right)\right]
		-\frac{1}{m}\sum_{r=1}^{m}\left[\sigma\left(\left\langle\mathbf{w}_{3-y,r}^{(t)},\mathbf{u}_{i,j}\right\rangle\right)\right]\\
		\overset{(a)}{\le}& \frac{1}{m}\sum_{r=1}^{m}\left[\sigma\left(\left\langle\mathbf{w}_{3-y,r}^{(t)}+\eta\mathbf{n}^{(t)}_{3-y,r},\mathbf{u}_{i,j}\right\rangle\right)\right]-\frac{1}{m}\sum_{r=1}^{m}\left[\sigma\left(\left\langle\mathbf{w}_{3-y,r}^{(t)},\mathbf{u}_{i,j}\right\rangle\right)\right]\\
		\overset{(b)}{\le}&  \frac{1}{m}\sum_{r=1}^{m}\left[\left|\left\langle\eta \mathbf{n}^{(t)},\mathbf{u}_{i,j}\right\rangle\right|\right]\\
		\overset{(c)}{\le}& \tilde{\mathcal{O}}(\eta\sigma_n\left\|\mathbf{u}_{i,j}\right\|_2),
	\end{aligned}
\end{equation}
where $(a)$ is obtained by the monotonicity of ReLU activation function; $(b)$ is because ReLU function is $1$-Lipschitz continuous; $(c)$ is due to Lemma \ref{lemma: half normal}.

For the term $A_2$, we have that with probability at least $1-\delta/T$, 
\begin{equation}
	\begin{aligned}
		A_2 \!=& \frac{1}{m}\!\sum_{r=1}^{m}\!\!\left[\!\sigma\!\left(\!\left\langle\!\mathbf{w}_{3-y,r}^{(t)}\!-\!\frac{\eta}{mB}\cdot\!\!\sum_{(\mathbf{x}_k,y_k)\in\mathcal{S}_{y}^{(t)}}\!\!\sigma'\left(\!\left\langle\!\mathbf{w}_{3-y,r}^{(t)},\boldsymbol{\xi}_k\!\right\rangle\!\right)\!\cdot h(C,\mathbf{x}_k,y_k) \cdot\text{prob}_{3-y}\!\left(\!\mathbf{W}^{(t)},\mathbf{x}_k\!\right)\!\cdot\!\boldsymbol{\xi}_k\right.\right.\right.\\
		&\!+\!\!\!\left.\left.\left.\frac{\eta}{mB}\cdot\!\!\!\sum_{(\mathbf{x}_k,y_k)\in\mathcal{S}_{3-y}^{(t)}}\!\!\!\sigma'\!\left(\!\left\langle\!\mathbf{w}_{3-y,r}^{(t)},\boldsymbol{\xi}_k\right\rangle\!\right)\!\cdot h(C,\mathbf{x},y)\!\cdot\!\left(\!1\!-\!\text{prob}_{3-y}\!\left(\!\mathbf{W}^{(t)},\mathbf{x}_k\!\right)\!\right)\cdot\boldsymbol{\xi}_k\!+\!\eta\mathbf{n}^{(t)}_{3-y,r},\boldsymbol{\xi}\!\right\rangle\!\!\right)\!\!\right]\\
		&-\!\frac{1}{m}\sum_{r=1}^{m}\left[\sigma\left(\left\langle\mathbf{w}_{3-y,r}^{(t)},\boldsymbol{\xi}\right\rangle\right)\right]\\
		\overset{(a)}{\le}& \frac{1}{m}\sum_{r=1}^{m}\left[\left|\left\langle\frac{\eta}{mB}\cdot\sum_{(\mathbf{x}_k,y_k)\in\mathcal{S}^{(t)}}\boldsymbol{\xi}_k,\boldsymbol{\xi}\right\rangle\right|+\left|\left\langle\eta\mathbf{n}_{3-y,r}^{(t)},\boldsymbol{\xi}\right\rangle\right|\right]\\
		\overset{(b)}{\le}& \tilde{\mathcal{O}}\left(\frac{\eta}{m\sqrt{B}}\sqrt{d}\sigma_p^2 +\eta \sqrt{d}\sigma_n\sigma_p\right),
	\end{aligned}
\end{equation}
where $(a)$ is because $\sigma'(\cdot)\ge 0, \text{prob}_y,\text{prob}_{3-y}\in[0,1]$ and ReLU function is $1$-Lipschitz continuous; $(b)$ is because of Lemma \ref{lemma: half normal}. 

For the term $A_3$, we have
\begin{equation}
	\begin{aligned}
		A_3 \!=& \frac{1}{m}\!\sum_{r=1}^{m}\!\sigma\!\!\left(\!\!
		\left\langle\! \mathbf{w}^{(t)}_{y,r} \!+\! \frac{\eta}{mB}\cdot\!\! \sum_{(\mathbf{x}_k,y_k)\in \mathcal{S}_{i,j}^{(t)}}\!\sigma'\!\left(\!\left\langle\mathbf{w}^{(t)}_{y,r}, \mathbf{u}_{i,j}\right\rangle\!\right)\cdot h(C,\mathbf{x}_k,y_k)\!\cdot\!\left(\!1\!-\!\text{prob}_y\!\left(\mathbf{W}^{(t)}, \mathbf{x}_k\right)\!\right)\!\cdot\!\mathbf{u}_{i,j}\right.\right.\\
		& \left.\left. +\eta\cdot\mathbf{n}^{(t)}_{y,r},\mathbf{u}_{i,j}\right\rangle\right)\!-\! \frac{1}{m}\sum_{r=1}^{m}\sigma\left(
		\left\langle \mathbf{w}^{(t)}_{y,r},\mathbf{u}_{i,j}\right\rangle\right)
	\end{aligned}
\end{equation}

Based on Lemma \ref{lemma: num concentration}, we can conclude that with probability at least $1-\exp(-2m)$, the number of activated neurons at iteration $t$ are at least $\frac{m}{4}$.
Then, with probability at least $1-\delta/T$, we have
\begin{equation}
	\begin{aligned}
		A_3 \!\ge& \frac{1}{m}\!\sum_{r=1}^{m}\!\sigma\!\left(\!
		\left\langle\! \mathbf{w}^{(t)}_{y,r} 
		\!+\! \frac{\eta}{mB} \!\!\sum_{(\mathbf{x}_k,y_k)\in \mathcal{S}_{i,j}^{(t)}}\!\sigma'\!\left(\left\langle\mathbf{w}^{(t)}_{y,r}, \mathbf{u}_{i,j}\right\rangle\right)\cdot h(C,\mathbf{x}_k,y_k)\cdot\left(1\!-\!\text{prob}_y\!\left(\mathbf{W}^{(t)}, \mathbf{x}_k\right)\!\right)\right.\right.\\
        &\left.\left.\mathbf{u}_{i,j},\mathbf{u}_{i,j}\right\rangle\right) - \frac{1}{m}\sum_{r=1}^{m}\left|\left\langle\eta\cdot \mathbf{n}_{y,r}^{(t)},\mathbf{u}_{i,j}\right\rangle\right|-\! \frac{1}{m}\sum_{r=1}^{m}\sigma\left(
		\left\langle \mathbf{w}^{(t)}_{y,r},\mathbf{u}_{i,j}\right\rangle\right)\\
		\ge& \Omega\!\left(\!\frac{\eta C}{Bm(\left\|\mathbf{u}_{i,j}\right\|_2+\sigma_p\sqrt{d})}\right)\!\!\sum_{(\mathbf{x}_k,y_k)\in \mathcal{S}_{i,j}^{(t)}}\!\left(1\!-\!\text{prob}_y\!\left(\mathbf{W}^{(t)}, \mathbf{x}_k\right)\!\right)\!\left\|\mathbf{u}_{i,j}\right\|_2^2 \!-\! \tilde{\mathcal{O}}\!\left(\eta\sigma_n\left\|\mathbf{u}_{i,j}\right\|_2\right),
	\end{aligned}
\end{equation}
The second inequality is by using the bound of the clipping multiplier.

Therefore, by Lemmas \ref{lemma: half normal} and \ref{lemma: sgd}, with probability at least $1 - 2\delta/T$, we have
\begin{equation}
\begin{aligned}
	A_3 \ge& \Omega\left(\frac{\eta\gamma_{i,j} C\left\|\mathbf{u}_{i,j}\right\|_2^2}{m(\left\|\mathbf{u}_{i,j}\right\|_2+\sigma_p\sqrt{d})}\right)\mathbb{E}_{(\mathbf{x}_k,y_k)\sim \mathcal{D}_{i,j}}\left[1-\text{prob}_{y_k}\left(\mathbf{W}^{(t)},\mathbf{x}_k\right)\right] - \mathcal{O}\left(\frac{\eta C}{m(\|\mathbf{u}_{i,j}\|_2+\sigma_p\sqrt{d})}\sqrt{\frac{1}{n}}\left\|\mathbf{u}_{i,j}\right\|_2^2\right)\\
	&- \tilde{\mathcal{O}}
	\left(\eta\sigma_n\left\|\mathbf{u}_{i,j}\right\|_2\right).
\end{aligned}
\end{equation}
In the following, we prove the bound of $A_4$.
Similar to the proof of the bound of $A_2$, we have that with probability at least $1-2\delta/T$, we have
\begin{equation}
	\begin{aligned}
		A_4 =& \frac{1}{m}\sum_{r=1}^{m}\left[\sigma\left(\left\langle\mathbf{w}_{y,r}^{(t)}-\frac{\eta}{mB}\sum_{(\mathbf{x}_k,y_k)\in\mathcal{S}_{3-y}^{(t)}}\sigma'\left(\left\langle\mathbf{w}_{y,r}^{(t)},\boldsymbol{\xi}_k\right\rangle\right)h(C,\mathbf{x}_k,y_k)\text{prob}_y\left(\mathbf{W}^{(t)},\mathbf{x}_k\right)\boldsymbol{\xi}_k\right.\right.\right.\\
		&+\left.\left.\left.\frac{\eta}{mB}\sum_{(\mathbf{x}_k,y_k)\in\mathcal{S}_{y}^{(t)}}\sigma'\left(\left\langle\mathbf{w}_{y,r}^{(t)},\boldsymbol{\xi}_k\right\rangle\right)h(C,\mathbf{x}_k,y_k)\left(1-\text{prob}_y\left(\mathbf{W}^{(t)},\mathbf{x}_k\right)\right)\boldsymbol{\xi}_k+\mathbf{n}^{(t)}_{y,r},\boldsymbol{\xi}\right\rangle\right)\right]\\
		&- \frac{1}{m}\sum_{r=1}^{m}\sigma\left(\left\langle\mathbf{w}_{y,r}^{(t)},\boldsymbol{\xi}\right\rangle\right)\\
		\ge& -\frac{1}{m}\sum_{r=1}^{m}\left[\left|\left\langle\frac{\eta}{mB}\sum_{(\mathbf{x}_k,y_k)\in\mathcal{S}^{(t)}}\boldsymbol{\xi}_k,\boldsymbol{\xi}\right\rangle\right|+\left|\left\langle\eta\mathbf{n}_{y,r}^{(t)},\boldsymbol{\xi}\right\rangle\right|\right]\\
		\ge& -\tilde{\mathcal{O}}\left(\frac{\eta}{m\sqrt{n}}\sqrt{d}\sigma_p^2 +\eta \sqrt{d}\sigma_n\sigma_p\right).
	\end{aligned}
\end{equation}

Substituting bounds of $A_1, A_2, A_3, A_4$ to (\ref{equ: delta}), we obtain the upper bound of $\Delta_{3-y}^{(t)}\left(\mathbf{x}\right) - \Delta_y^{(t)}\left(\mathbf{x}\right)$.
With probability at least $1-4\delta/T$,
\begin{equation}\label{equ: bound_Delta}
	\begin{aligned}
		&\Delta_{3-y}^{(t)}\left(\mathbf{x}\right) - \Delta_y^{(t)}\left(\mathbf{x}\right) \\
		\le&
		-\underbrace{\Omega\left(\frac{\eta\gamma_{i,j} C\left\|\mathbf{u}_{i,j}\right\|_2^2}{m(\left\|\mathbf{u}_{i,j}\right\|_2+\sigma_p\sqrt{d})}\right)\cdot\left[\mathbb{E}_{(\mathbf{x}_k,y_k)\sim \mathcal{D}_{i,j}}\left[1-\text{prob}_{y_k}\left(\mathbf{W}^{(t)},\mathbf{x}_k\right)\right]\right]}_{\Phi_1}\\  &+ \underbrace{\tilde{\mathcal{O}}\left(\frac{\eta \Lambda_{i,j}}{m}\sqrt{\frac{1}{n}}\left\|\mathbf{u}_{i,j}\right\|_2^2\right)
		+ \mathcal{O}\left(\eta\sigma_n\left\|\mathbf{u}_{i,j}\right\|_2\right)
		+\tilde{\mathcal{O}}\left(\frac{\eta}{m\sqrt{n}}\sqrt{d}\sigma_p^2 +\eta \sqrt{d}\sigma_n\sigma_p\right)}_{\Phi_2}.
	\end{aligned}
\end{equation}
Armed with the loss increment bound (\ref{equ: bound_Delta}), we prove the test loss bound of each data $(\mathbf{x},y)\sim \mathcal{D}$ in the next subsection.
\subsubsection{Test Loss Bound}
Under Assumption \ref{assumption: non-perfect}, for any $(\mathbf{x},y)\sim \mathcal{D}$, we have
\begin{align}
	1 - \text{prob}_y\left(\mathbf{W}^{(t)},\mathbf{x}\right) \ge 1-\exp(-s).
\end{align}

By (\ref{equ: loss}) and 
Lemma \ref{lemma: Gamma}, with probability at least $1-4\delta/T$, we can upper bound the loss increment by a piece-wise linear function,
\begin{equation}
	\begin{aligned}
		&\mathcal{L}_{\mathcal{D}_{i,j}}\left(\mathbf{W}^{(t+1)}\right) - \mathcal{L}_{\mathcal{D}_{i,j}}\left(\mathbf{W}^{(t)}\right)\\ =&\mathbb{E}_{(\mathbf{x},y)\sim \mathcal{D}_{i,j}}\left[ \log\left(1+\left(1-\text{prob}_{y}\left(\mathbf{W}^{(t)},\mathbf{x}\right)\right)\cdot\left(\exp\left(\Delta_{3-y}^{(t)}\left(\mathbf{x}\right)-\Delta_y^{(t)}\left(\mathbf{x}\right)\right)-1\right) \right)\right]\\
		\overset{(a)}{\le}&- \mathbb{E}_{(\mathbf{x},y)\sim \mathcal{D}_{i,j}}\left[\Gamma\left(\Delta_{3-y}^{(t)}\left(\mathbf{x}\right)-\Delta_y^{(t)}\left(\mathbf{x}\right)\right)\cdot \Phi_1\right]+\mathbb{E}_{(\mathbf{x},y)\sim \mathcal{D}_{i,j}}\left[\Gamma\left(\Delta_{3-y}^{(t)}\left(\mathbf{x}\right)-\Delta_y^{(t)}\left(\mathbf{x}\right)\right)\cdot\Phi 2\right]\\
	\end{aligned}
\end{equation}
where $(a)$ is by Lemma \ref{lemma: Gamma} and Lemma  $(\ref{equ: bound_Delta})$.
Then, substituting (\ref{equ: bound_Delta}) to the above inequality yields
\begin{equation}
	\begin{aligned}
		\mathcal{L}_{\mathcal{D}_{i,j}}\left(\mathbf{W}^{(t+1)}\right) \!-\! \mathcal{L}_{\mathcal{D}_{i,j}}\left(\mathbf{W}^{(t)}\right)\!	\overset{(a)}{\le}& \!-\!\Omega\left(\frac{\eta\gamma_{i,j} \Lambda_{i,j}\left\|\mathbf{u}_{i,j}\right\|_2^2}{m}\right)\cdot \mathcal{L}_{\mathcal{D}_{i,j}}\left(\mathbf{W}^{(t)}\right)\!+\!\tilde{\mathcal{O}}\left(\frac{\eta \Lambda_{i,j}}{m}\sqrt{\frac{1}{n}}\left\|\mathbf{u}_{i,j}\right\|_2^2\right)\\
		&+ \tilde{\mathcal{O}}\left(\eta\sigma_n\left\|\mathbf{u}_{i,j}\right\|_2\right)+\tilde{\mathcal{O}}\left(\frac{\eta}{m\sqrt{n}}\sqrt{d}\sigma_p^2 +\eta \sqrt{d}\sigma_n\sigma_p\right),
	\end{aligned}
\end{equation}
where $(a)$ is obtain by Lemma \ref{lemma: -logx and 1-x}.
$\underline{\Gamma} = \frac{-\log(1+(1-\exp(-s))\cdot(\exp(-a)-1))}{a}, \gamma = -\frac{\exp(-s)}{\ln(1-\exp(-s))}$ and $a$ is the lower bound of $\Delta_{3-y}^{(t)}-\Delta_y^{(t)}$ (By Lemma \ref{lemma: update_bounded}, $\Delta_{3-y}^{(t)}-\Delta_y^{(t)}$ is lower bounded by a constant).
Therefore, we have
\begin{equation}
	\begin{aligned}
		\mathcal{L}_{\mathcal{D}_{i,j}}\left(\mathbf{W}^{(t+1)}\right) \le& \left(1-\Omega\left(\frac{\eta\gamma_{i,j}\Lambda_{i,j}\left\|\mathbf{u}_{i,j}\right\|_2^2}{m}\right)\right)\cdot\mathcal{L}_{\mathcal{D}_{i,j}}\left(\mathbf{W}^{(t)}\right) + \tilde{\mathcal{O}}\left(\frac{\eta \Lambda_{i,j}}{m}\sqrt{\frac{1}{n}}\left\|\mathbf{u}_{i,j}\right\|_2^2\right)\\
		&+ \tilde{\mathcal{O}}\left(\eta\sigma_n\left\|\mathbf{u}_{i,j}\right\|_2\right)+\tilde{\mathcal{O}}\left(\frac{\eta}{m\sqrt{n}}\sqrt{d}\sigma_p^2 +\eta \sqrt{d}\sigma_n\sigma_p\right).
	\end{aligned}
\end{equation}
Then, combining all $T$ iterations and using Lemma \ref{lemma: geometric}, we have
\begin{equation}
	\begin{aligned}
		&\mathcal{L}_{\mathcal{D}_{i,j}}\left(\mathbf{W}^{(T)}\right)\\
		\le& \left(1-\Omega\left(\frac{\eta\gamma_{i,j} \Lambda_{i,j}\left\|\mathbf{u}_{i,j}\right\|_2^2}{m}\right)\right)^T\!\mathcal{L}_{\mathcal{D}_{i,j}}\left(\mathbf{W}^{(0)}\right)\!+\! \left(\tilde{\mathcal{O}}\!\left(\frac{\eta\Lambda_{i,j}}{m}\sqrt{\frac{1}{n}}\left\|\mathbf{u}_{i,j}\right\|_2^2\right)\!+\! \tilde{\mathcal{O}}\left(\eta\sigma_n\left\|\mathbf{u}_{i,j}\right\|_2\right)\right.\\
		+&\left.\tilde{\mathcal{O}}\left(\frac{\eta}{m\sqrt{n}}\sqrt{d}\sigma_p^2 +\eta \sqrt{d}\sigma_n\sigma_p\right)\right)\cdot\mathcal{O}\left(\frac{m}{\eta \gamma_{i,j}\Lambda_{i,j}\left\|\mathbf{u}_{i,j}\right\|_2^2}\right)\\
		\le& \exp\left(-\Omega\left(\frac{T\eta\gamma_{i,j}\Lambda_{i,j}}{m}\left\|\mathbf{u}_{i,j}\right\|_2^2\right)\right)\mathcal{L}_{\mathcal{D}_{i,j}}\left(\mathbf{W}^{(0)}\right)+\tilde{\mathcal{O}}\left(\sqrt{\frac{1}{n}}\frac{1}{\gamma_{i,j}}\right)\\
		+& \tilde{\mathcal{O}}\!\left(\!\frac{m\sigma_n}{\gamma_{i,j}\Lambda_{i,j}\left\|\mathbf{u}_{i,j}\right\|_2}\!+\!\frac{1}{\sqrt{n}}\frac{\sqrt{d}\sigma_p^2}{\gamma_{i,j}\Lambda_{i,j}\left\|\mathbf{u}_{i,j}\right\|_2^2}\!+\!\frac{ m\sqrt{d}\sigma_n\sigma_p}{\gamma_{i,j}\Lambda_{i,j}\left\|\mathbf{u}_{i,j}\right\|_2^2}\!\right).
	\end{aligned}
\end{equation}
Setting the parameters with Condition \ref{condition} yields the conclusion.
This completes the proof.

\subsection{Proof of Theorem \ref{theorem: lower bound}}
\begin{proof}
Recall that in (\ref{equ: loss}), we have
\begin{equation}\label{equ: loss_increment_lower}
	\begin{aligned}
		&\mathcal{L}(\mathbf{W}^{(t+1)},\mathbf{x},y) - \mathcal{L}(\mathbf{W}^{(t)},\mathbf{x},y)\\
		=	&\log\left(1+ \left(1-\text{prob}_y\left(\mathbf{W}^{(t)},\mathbf{x}\right)\right)\left(\exp\left(\Delta_{3-y}^{(t)}\left(\mathbf{x}\right)-\Delta_y^{(t)}\left(\mathbf{x}\right)\right)-1\right)\right)\\
		\overset{(a)}{\ge}& c_0^{(t)} \cdot \left(1-\text{prob}_y\left(\mathbf{W}^{(t)},\mathbf{x}\right)\right)\left(\exp\left(\Delta_{3-y}^{(t)}\left(\mathbf{x}\right)-\Delta_y^{(t)}\left(\mathbf{x}\right)\right)-1\right)\\
        \overset{(b)}{=}& \Omega\left(\exp\left(\Delta_{3-y}^{(t)}\left(\mathbf{x}\right)-\Delta_y^{(t)}\left(\mathbf{x}\right)\right)-1\right),
	\end{aligned}
\end{equation}
where $c_0^{(t)}>0$ for any $t\in[0,T-1]$ are constants.
Here $(a)$ is obtained from Lemma \ref{lemma: log lower bound}
Then, we bound $\Delta_{3-y}^{(t)}\left(\mathbf{x}\right)-\Delta_y^{(t)}\left(\mathbf{x}\right)$; $(b)$ is by $1-\exp(-s) \le 1-\text{prob}_y(\mathbf{W},\mathbf{x}) \le 1$ with Assumption \ref{assumption: non-perfect}.
Next, we will prove a lower bound of $\Delta_{3-y}^{(t)}\left(\mathbf{x}\right)-\Delta_y^{(t)}\left(\mathbf{x}\right)$.
Recall that in (\ref{equ: delta}), we have
\begin{equation}
	\begin{aligned}
		&\Delta_{3-y}^{(t)}\left(\mathbf{x}\right)-\Delta_y^{(t)}\left(\mathbf{x}\right)\\
		=& \underbrace{\frac{1}{m}\sum_{r=1}^{m}\left[\sigma\!\left(
			\!\left\langle \mathbf{w}^{(t+1)}_{3-y,r},\mathbf{u}_{i,j}\right\rangle\!\right) \!-\! \sigma\!\left(\!
			\left\langle \mathbf{w}^{(t)}_{3-y,r},\mathbf{u}_{i,j}\right\rangle\!\right)\right]}_{A_1} \!+\! \underbrace{\frac{1}{m}\sum_{r=1}^{m}\left[\sigma\!\left(\!
			\left\langle \mathbf{w}^{(t+1)}_{3-y,r},\boldsymbol{\xi}\right\rangle\! \right) \!-\! \sigma\!\left(\!
			\left\langle \mathbf{w}^{(t)}_{3-y,r},\boldsymbol{\xi}\right\rangle \!\right)\right]}_{A_2}\\
		&-\! \underbrace{\frac{1}{m}\sum_{r=1}^{m}\left[\sigma\!\left(
			\left\langle \mathbf{w}^{(t+1)}_{y,r},\mathbf{u}_{i,j}\right\rangle\right) \!-\! \sigma\!\left(
			\left\langle \mathbf{w}^{(t)}_{y,r},\mathbf{u}_{i,j}\right\rangle\right)\right]}_{A_3} \!-\! \underbrace{\frac{1}{m}\sum_{r=1}^{m}\left[\sigma\left(
			\left\langle \mathbf{w}^{(t+1)}_{y,r},\boldsymbol{\xi}\right\rangle \right) \!-\! \sigma\!\left(
			\left\langle \mathbf{w}^{(t)}_{y,r},\boldsymbol{\xi}\right\rangle \right)\right]}_{A_4},
	\end{aligned}
\end{equation}
We then find the lower bounds of $A_1, A_2$ and upper bounds of $A_3, A_4$ to obtain the lower bound of $\Delta_{3-y}^{(t)}\left(\mathbf{x}\right)-\Delta_y^{(t)}\left(\mathbf{x}\right)$.

For the term $A_1$, with probability at least $1-\delta/T$, we have
\begin{equation}
	\begin{aligned}
		A_1\! =& \frac{1}{m}\!\sum_{r=1}^{m}\!\left[\!\sigma\!\left(\!\left\langle\!\mathbf{w}_{3-y,r}^{(t)}\!\!-\!\frac{\eta}{mB}\cdot\!\!\!\sum_{(\mathbf{x}_k,y_k)\in\mathcal{S}^{(t)}_{i,j}}\!\!\sigma'\!\left(\!\left\langle\!\mathbf{w}_{3-y,r}^{(t)},\mathbf{u}_{i,j}\!\right\rangle\!\right)\!\cdot\!h(C,\mathbf{x}_k,y_k)\!\cdot\!\text{prob}_{3-y}\!\left(\!\mathbf{W}^{(t)},\mathbf{x}_k\!\right)\!\cdot\!\mathbf{u}_{i,j}\right.\right.\right.\\
		&+\left.\left.\left.\eta\cdot\mathbf{n}^{(t)}_{3-y,r},\mathbf{u}_{i,j}\right\rangle\right)\right]
		-\frac{1}{m}\sum_{r=1}^{m}\left[\sigma\left(\left\langle\mathbf{w}_{3-y,r}^{(t)},\mathbf{u}_{i,j}\right\rangle\right)\right]\\
		\overset{(a)}{\ge}& -\frac{\eta}{mB}\cdot\!\!\!\sum_{(\mathbf{x}_k,y_k)\in\mathcal{S}^{(t)}_{i,j}}\!\!\sigma'\!\left(\!\left\langle\!\mathbf{w}_{3-y,r}^{(t)},\mathbf{u}_{i,j}\!\right\rangle\!\right)\!\cdot h(C,\mathbf{x}_k,y_k)\cdot\text{prob}_{3-y}\!\left(\!\mathbf{W}^{(t)},\mathbf{x}_k\!\right)\!\cdot\left\|\mathbf{u}_{i,j}\right\|_2^2\\
        &+ \frac{1}{m}\sum_{r=1}^m\langle\eta\cdot\mathbf{n}^{(t)}_{3-y,r}, \mathbf{u}_{i,j}\rangle,\\
		\overset{(b)}{\ge}& -\frac{\eta}{mB}\sum_{(\mathbf{x}_k,y_k)\in\mathcal{S}^{(t)}_{i,j}} \left(1-\text{prob}_{y}\!\left(\!\mathbf{W}^{(t)},\mathbf{x}_k\!\right)\right)\!\cdot\left\|\mathbf{u}_{i,j}\right\|_2^2 + \langle\eta\cdot\mathbf{n}^{(t)}_{3-y,r}, \mathbf{u}_{i,j}\rangle\\
		\overset{(c)}{\ge}& -\frac{\eta\gamma_{i,j}}{m}\mathbb{E}_{(\mathbf{x}_k,y_k)\sim\mathcal{D}_{i,j}}\! \left[\!1\!-\!\text{prob}_{y}\!\left(\!\mathbf{W}^{(t)},\mathbf{x}_k\!\right)\!\right]\!\cdot\!\left\|\mathbf{u}_{i,j}\right\|_2^2 \!-\! \tilde{\mathcal{O}}\left(\frac{\eta}{m}\sqrt{\frac{1}{B}}\left\|\mathbf{u}_{i,j}\right\|_2^2\right)\!+\! \frac{\eta}{m}\sum_{r=1}^{m}\langle\mathbf{n}^{(t)}_{3-y,r}, \mathbf{u}_{i,j}\rangle,
	\end{aligned}
\end{equation}
where $(a)$ is due to the Condition \ref{condition} that $\left\|\mathbf{u}_{i,j}\right\|_2 = \Omega(\sigma_n)$.
$(b)$ is by the fact that $\sigma'(\cdot), h(C,\mathbf{x}_k,y_k) \le1$; $(c)$ is by Lemma \ref{lemma: sgd}.

For the term $A_2$, with probability at least $1-\delta/T$, we have
\begin{equation}
	\begin{aligned}
		A_2 =& \frac{1}{m}\!\sum_{r=1}^{m}\!\left[\!\sigma\!\left(\!\left\langle\!\mathbf{w}_{3-y,r}^{(t)}\!-\!\frac{\eta}{mB}\cdot\!\!\sum_{(\mathbf{x}_k,y_k)\in\mathcal{S}_{y}^{(t)}}\!\!\sigma'\left(\!\left\langle\!\mathbf{w}_{3-y,r}^{(t)},\boldsymbol{\xi}_k\!\right\rangle\!\right)\!\cdot\!h(C,\mathbf{x}_k,y_k) \cdot\text{prob}_{3-y}\!\left(\!\mathbf{W}^{(t)},\mathbf{x}_k\!\right)\!\cdot\!\boldsymbol{\xi}_k\right.\right.\right.\\
		&\!+\!\!\!\left.\left.\left.\frac{\eta}{mB}\cdot\!\!\!\sum_{(\mathbf{x}_k,y_k)\in\mathcal{S}_{3-y}^{(t)}}\!\!\!\sigma'\!\left(\!\left\langle\!\mathbf{w}_{3-y,r}^{(t)},\boldsymbol{\xi}_k\right\rangle\!\right)\!\cdot\!h(C,\mathbf{x},y)\cdot\left(1\!-\!\text{prob}_{3-y}\!\left(\!\mathbf{W}^{(t)},\mathbf{x}_k\!\right)\!\right)\cdot\!\boldsymbol{\xi}_k\!+\!\eta\mathbf{n}^{(t)}_{3-y,r},\boldsymbol{\xi}\!\right\rangle\!\!\right)\!\right]\\
		&-\!\frac{1}{m}\sum_{r=1}^{m}\left[\sigma\left(\left\langle\mathbf{w}_{3-y,r}^{(t)},\boldsymbol{\xi}\right\rangle\right)\right]\\
		\ge& \!-\!\frac{\eta}{mB} \!\sum_{(\mathbf{x}_k,y_k)\in\mathcal{S}_{y}^{(t)}}\!\!\sigma'\left(\!\left\langle\!\mathbf{w}_{3-y,r}^{(t)},\boldsymbol{\xi}_k\!\right\rangle\!\right)\!\cdot\!h(C,\mathbf{x}_k,y_k) \cdot\text{prob}_{3-y}\!\left(\!\mathbf{W}^{(t)},\mathbf{x}_k\!\right)\!\cdot\!|\langle\boldsymbol{\xi}_k, \boldsymbol{\xi}\rangle| \!+\!\langle \eta\mathbf{n}^{(t)}_{3-y,r}, \boldsymbol{\xi}\rangle\\
		\ge& -\tilde{\mathcal{O}}\left(\frac{\eta}{m\sqrt{B}}\sqrt{d}\sigma_p^2\right) +  \frac{1}{m}\sum_{r=1}^{m}\eta\langle\mathbf{n}^{(t)}_{3-y,r}, \boldsymbol{\xi}\rangle ,
	\end{aligned}
\end{equation}

For the term $A_3$, with probability at least $1-\delta/T$, we have
\begin{equation}
	\begin{aligned}
		A_3 =& \frac{1}{m}\sum_{r=1}^{m}\!\sigma\!\left(\!
		\left\langle\! \mathbf{w}^{(t)}_{y,r} \!+\! \frac{\eta}{mB}\cdot\!\! \sum_{(\mathbf{x}_k,y_k)\in \mathcal{S}_{i,j}^{(t)}}\!\!\sigma'\!\left(\!\left\langle\mathbf{w}^{(t)}_{y,r}, \mathbf{u}_{i,j}\right\rangle\!\right)\!\cdot\! h(C,\mathbf{x}_k,y_k)\!\cdot\!\left(\!1\!-\!\text{prob}_y\!\left(\mathbf{W}^{(t)}, \mathbf{x}_k\right)\!\right)\!\cdot\!\mathbf{u}_{i,j}\right.\right.\\
		& \left.\left. +\eta\cdot\mathbf{n}^{(t)}_{y,r},\mathbf{u}_{i,j}\right\rangle\right)\!-\! \frac{1}{m}\sum_{r=1}^{m}\sigma\left(
		\left\langle \mathbf{w}^{(t)}_{y,r},\mathbf{u}_{i,j}\right\rangle\!\right)\\
        \overset{(a)}{\le}& \frac{\eta}{mB}\cdot\!\!\!\sum_{(\mathbf{x}_k,y_k)\in\mathcal{S}^{(t)}_{i,j}}\!\!\sigma'\!\left(\!\left\langle\!\mathbf{w}_{3-y,r}^{(t)},\mathbf{u}_{i,j}\!\right\rangle\!\right)\!\cdot h(C,\mathbf{x}_k,y_k)\cdot\text{prob}_{3-y}\!\left(\!\mathbf{W}^{(t)},\mathbf{x}_k\!\right)\!\cdot\left\|\mathbf{u}_{i,j}\right\|_2^2\\
        &+ \frac{1}{m}\sum_{r=1}^m\langle\eta\cdot\mathbf{n}^{(t)}_{3-y,r}, \mathbf{u}_{i,j}\rangle,\\
		\overset{(b)}{\le}& \frac{\eta}{mB}\sum_{(\mathbf{x}_k,y_k)\in\mathcal{S}^{(t)}_{i,j}} \left(1-\text{prob}_{y}\!\left(\!\mathbf{W}^{(t)},\mathbf{x}_k\!\right)\right)\!\cdot\left\|\mathbf{u}_{i,j}\right\|_2^2 + \langle\eta\cdot\mathbf{n}^{(t)}_{3-y,r}, \mathbf{u}_{i,j}\rangle\\
		\overset{(c)}{\le}& \frac{\eta\gamma_{i,j}}{m}\mathbb{E}_{(\mathbf{x}_k,y_k)\sim\mathcal{D}_{i,j}}\! \left[\!1\!-\!\text{prob}_{y}\!\left(\!\mathbf{W}^{(t)},\mathbf{x}_k\!\right)\!\right]\!\cdot\!\left\|\mathbf{u}_{i,j}\right\|_2^2 \!-\! \tilde{\mathcal{O}}\left(\frac{\eta}{m}\sqrt{\frac{1}{B}}\left\|\mathbf{u}_{i,j}\right\|_2^2\right)\!+\! \frac{\eta}{m}\sum_{r=1}^{m}\langle\mathbf{n}^{(t)}_{3-y,r}, \mathbf{u}_{i,j}\rangle,
	\end{aligned}
\end{equation}
where the inequality is by Condition \ref{condition}, which implies $\left\|\mathbf{u}_{i,j}\right\|_2 = \Omega(\sigma_n)$.

For the term $A_4$, with probability at least $1-\delta/T$, we have
\begin{equation}
	\begin{aligned}
		A_4 =& \frac{1}{m}\sum_{r=1}^{m}\left[\sigma\left(\left\langle\mathbf{w}_{y,r}^{(t)}-\frac{\eta}{mB}\sum_{(\mathbf{x}_k,y_k)\in\mathcal{S}_{3-y}^{(t)}}\sigma'\left(\left\langle\mathbf{w}_{y,r}^{(t)},\boldsymbol{\xi}_k\right\rangle\right)h(C,\mathbf{x}_k,y_k)\text{prob}_y\left(\mathbf{W}^{(t)},\mathbf{x}_k\right)\boldsymbol{\xi}_k\right.\right.\right.\\
		&+\!\left.\left.\left.\frac{\eta}{mB}\sum_{(\mathbf{x}_k,y_k)\in\mathcal{S}_{y}^{(t)}}\sigma'\left(\left\langle\mathbf{w}_{y,r}^{(t)},\boldsymbol{\xi}_k\right\rangle\right)h(C,\mathbf{x}_k,y_k)\left(1-\text{prob}_y\left(\mathbf{W}^{(t)},\mathbf{x}_k\right)\right)\boldsymbol{\xi}_k+\mathbf{n}^{(t)}_{y,r},\boldsymbol{\xi}\right\rangle\!\right)\!\right]\\
		&- \frac{1}{m}\sum_{r=1}^{m}\sigma\left(\left\langle\mathbf{w}_{y,r}^{(t)},\boldsymbol{\xi}\right\rangle\right)\\
		\le& \tilde{\mathcal{O}}\left(\frac{\eta}{m\sqrt{n}}\sqrt{d}\sigma_p^2\right) + \frac{1}{m}\sum_{r=1}^{m}\eta \langle\mathbf{n}^{(t)}_{y,r},\boldsymbol{\xi} \rangle,
	\end{aligned}
\end{equation}
where the inequality is by Condition \ref{condition} that $\sigma_n = \mathcal{O}(\sigma_p)$, $h(C,\mathbf{x}_k,y_k)<1$.

Combining the bounds together, we have
\begin{equation}\label{equ: Delta_lower}
	\begin{aligned}
		&\Delta_{3-y}^{(t)}\left(\mathbf{x}\right)-\Delta_y^{(t)}\left(\mathbf{x}\right) \ge -\frac{\eta\gamma_{i,j}}{m}\mathbb{E}_{(\mathbf{x}_k,y_k)\sim\mathcal{D}_{i,j}} \left[1-\text{prob}_{y}\!\left(\!\mathbf{W}^{(t)},\mathbf{x}_k\!\right)\right]\!\cdot\left\|\mathbf{u}_{i,j}\right\|_2^2\\
        &+ \frac{1}{m}\sum_{r=1}^{m}\langle\eta\cdot\mathbf{n}^{(t)}_{3-y,r}, \mathbf{u}_{i,j}\rangle + \frac{1}{m}\sum_{r=1}^{m}\langle \eta\mathbf{n}^{(t)}_{3-y,r}, \boldsymbol{\xi}\rangle-\tilde{\mathcal{O}}\left(\frac{\eta}{m\sqrt{n}}\sqrt{d}\sigma_p^2\right)\\
        &- \frac{1}{m}\sum_{r=1}^{m}\eta \langle\mathbf{n}^{(t)}_{y,t},\boldsymbol{\xi}\rangle -\tilde{\mathcal{O}}\left(\frac{\eta}{m}\sqrt{\frac{1}{n}}\left\|\mathbf{u}_{i,j}\right\|_2^2\right),
	\end{aligned}
\end{equation}
with probability at least $1-4\delta/T$.
Substituting (\ref{equ: Delta_lower}) into (\ref{equ: loss_increment_lower}), we have
\begin{equation}
	\begin{aligned}
		&\mathbb{E}[\mathcal{L}(\mathbf{W}^{(t+1)},\mathbf{x},y)] - \mathcal{L}(\mathbf{W}^{(t)},\mathbf{x},y)\\
		\ge& \mathbb{E}\left[\Omega\left(\exp\left(\Delta_{3-y}^{(t)}\left(\mathbf{x}\right)-\Delta_y^{(t)}\left(\mathbf{x}\right)\right)-1\right)\right]\\
		\ge& \Omega\left(\mathbb{E}\left[ \exp\left(-\frac{\eta\gamma_{i,j}}{m}\mathbb{E}_{(\mathbf{x}_k,y_k)\sim\mathcal{D}_{i,j}} \left[1-\text{prob}_{y}\!\left(\!\mathbf{W}^{(t)},\mathbf{x}_k\!\right)\right]\!\cdot\left\|\mathbf{u}_{i,j}\right\|_2^2 + \frac{1}{m}\sum_{r=1}^{m}\langle \eta\mathbf{n}^{(t)}_{3-y,r}, \boldsymbol{\xi}\rangle \right.\right.\right.\\
		&\left.\left.\left.\!\!\! + \frac{1}{m}\sum_{r=1}^{m}\langle\eta\cdot\mathbf{n}^{(t)}_{3-y,r}, \mathbf{u}_{i,j}\rangle - \frac{1}{m}\sum_{r=1}^{m}\eta \langle\mathbf{n}^{(t)}_{y,t},\boldsymbol{\xi}\rangle -  \tilde{\mathcal{O}}\left(\frac{\eta}{m\sqrt{n}}\sqrt{d}\sigma_p^2+\frac{\eta}{m}\sqrt{\frac{1}{n}}\left\|\mathbf{u}_{i,j}\right\|_2^2\right)\right)-1\right]\right).
		\end{aligned}
\end{equation}
Here with a probability at least $1-\delta/T$, we have
\begin{equation}
\begin{aligned}
	&\mathbb{E}\left[\exp\left(\frac{1}{m}\sum_{r=1}^{m}\langle\eta\mathbf{n}^{(t)}_{3-y,r},\mathbf{u}_{i,j}\rangle + \frac{1}{m}\sum_{r=1}^{m}\langle \eta\mathbf{n}^{(t)}_{3-y,r}, \boldsymbol{\xi}\rangle - \frac{1}{m}\sum_{r=1}^{m}\eta \langle\mathbf{n}^{(t)}_{y,t},\boldsymbol{\xi}\rangle\right)\right]\\
	=& \exp\left(\eta^2\frac{\left\|\mathbf{u}_{i,j}\right\|_2^2\sigma_n^2 + 2\left\|\boldsymbol{\xi}\right\|_2^2\sigma_n^2}{2m}\right)\\
 =& \exp \left(\tilde{\Theta}\left(\eta^2 \frac{\left\|\mathbf{u}_{i,j}\right\|_2^2\sigma_n^2 + \sigma_p^2d\sigma_n^2}{2m}\right)\right),
\end{aligned}
\end{equation}
where the least equality is by Lemma \ref{lemma: squared_norm}.
With a probability at least $1- \delta/T$, we have
\begin{equation}
	\begin{aligned}		
		&\mathbb{E}[\mathcal{L}(\mathbf{W}^{(t+1)},\mathbf{x},y)] - \mathcal{L}(\mathbf{W}^{(t)},\mathbf{x},y)\\
		\ge&\Omega\left(-\frac{\eta\gamma_{i,j}}{m}\mathbb{E}_{(\mathbf{x}_k,y_k)\sim\mathcal{D}_{i,j}}\left[1-\text{prob}_{y}\!\left(\!\mathbf{W}^{(t)},\mathbf{x}_k\!\right)\right]\!\cdot\left\|\mathbf{u}_{i,j}\right\|_2^2\right)\\
		&+\tilde{\Omega}\left(\frac{\eta^2\sigma_n^2\left\|\mathbf{u}_{i,j}\right\|^2_2}{2m}\right) + \tilde{\Omega}\left(\frac{\eta^2d\sigma_n^2\sigma_p^2}{2m}\right) - \tilde{\mathcal{O}}\left(\frac{\eta}{m\sqrt{n}}\sqrt{d}\sigma_p^2\right)-\tilde{\mathcal{O}}\left(\frac{\eta}{m}\sqrt{\frac{1}{n}}\left\|\mathbf{u}_{i,j}\right\|_2^2\right)\\
		\ge& -\mathcal{O}\left( \frac{\eta\gamma_{i,j}}{m}\left\|\mathbf{u}_{i,j}\right\|^2_2\right)\mathcal{L}_{\mathcal{D}_{i,j}}\left(\mathbf{W}^{(t)}\right)-\tilde{\mathcal{O}}\left(\frac{\eta}{m}\sqrt{\frac{1}{n}}\left\|\mathbf{u}_{i,j}\right\|_2^2\right)\\
		&+ \tilde{\Omega}\left(\frac{\eta^2\sigma_n^2\left\|\mathbf{u}_{i,j}\right\|^2_2}{2m}\right) +  \tilde{\Omega}\left(\frac{\eta^2d\sigma_n^2\sigma_p^2}{2m}\right)- \tilde{\mathcal{O}}\left(\frac{\eta}{m\sqrt{n}}\sqrt{d}\sigma_p^2\right).
	\end{aligned},
\end{equation}
where the second equality is by Lemma \ref{lemma: -logx and 1-x} (the $(1-\text{prob}_y(\mathbf{W}^{(t)},\mathbf{x}))$ is almost surely lower bounded).
Then, with a probability at least $1-5\delta/T$, we have
\begin{equation}
	\begin{aligned}
		\mathbb{E}_{\mathbf{n}^{(t)}}[\mathcal{L}_{\mathcal{D}_{i,j}}(\mathbf{W}^{(t+1)})]\ge& \left(1-\mathcal{O}\left(\frac{\eta\gamma_{i,j}\left\|\mathbf{u}_{i,j}\right\|^2_2}{m}\right)\right) \mathcal{L}_{\mathcal{D}_{i,j}}(\mathbf{W}^{(t)}) -\tilde{\mathcal{O}}\left(\frac{\eta}{m}\sqrt{\frac{1}{n}}\left\|\mathbf{u}_{i,j}\right\|_2^2\right)\\
		&+ \tilde{\Omega}\left(\frac{\eta^2\sigma_n^2\left\|\mathbf{u}_{i,j}\right\|^2_2}{2m}\right) +  \tilde{\Omega}\left(\frac{\eta^2d\sigma_n^2\sigma_p^2}{2m}\right)- \tilde{\mathcal{O}}\left(\frac{\eta}{m\sqrt{n}}\sqrt{d}\sigma_p^2\right).
	\end{aligned}
\end{equation}
Combining all the iterations, with a probability at least $1-5\delta$, we have
\begin{equation}
	\begin{aligned}
		\mathbb{E}_{\mathbf{n}^{(0)},\cdots,\mathbf{n}^{(T-1)}}[\mathcal{L}_{\mathcal{D}_{i,j}}(\mathbf{W}^{(T)})]
		\ge&\left(1-\mathcal{O}\left(\frac{\eta\gamma_{i,j}\left\|\mathbf{u}_{i,j}\right\|_2^2}{m}\right)\right)^T\mathcal{L}_{\mathcal{D}_{i,j}}(\mathbf{W}^{(0)}) + \frac{1-\left(1-\mathcal{O}\left(\frac{\eta\gamma_{i,j}\left\|\mathbf{u}_{i,j}\right\|_2^2}{m}\right)\right)^T}{\mathcal{O}\left(\frac{\eta\gamma_{i,j}\left\|\mathbf{u}_{i,j}\right\|_2^2}{m}\right)}\\
		&  \left[\tilde{\Omega}\left(\frac{\eta^2\sigma_n^2\left\|\mathbf{u}_{i,j}\right\|^2_2}{m}\right) + \tilde{\Omega}\left(\frac{\eta^2d\sigma_n^2\sigma_p^2}{m}\right)-\tilde{\mathcal{O}}\left(\frac{\eta}{m\sqrt{n}}\sqrt{d}\sigma_p^2 +\frac{\eta}{m}\sqrt{\frac{1}{n}}\left\|\mathbf{u}_{i,j}\right\|_2^2\right)\right].
	\end{aligned}
\end{equation}
With the number of iterations $T\ge \Omega\left(-\frac{1}{\log\left(1-\Omega\left(\frac{\eta\gamma_{i,j}\left\|\mathbf{u}_{i,j}\right\|_2^2}{m}\right)\right)}\right)$ and a probability at least $1-\delta$, we have
\begin{equation}
	\begin{aligned}
		\mathbb{E}_{\mathbf{n}^{(0)},\cdots,\mathbf{n}^{(T-1)}}[\mathcal{L}_{\mathcal{D}_{i,j}}(\mathbf{W}^{(T)})]
		\ge&\left(1-\mathcal{O}\left(\frac{\eta\gamma_{i,j}\left\|\mathbf{u}_{i,j}\right\|_2^2}{m}\right)\right)^T\mathcal{L}_{\mathcal{D}_{i,j}}(\mathbf{W}^{(0)}) + \tilde{\Omega}\left(\frac{m}{\eta\gamma_{i,j}\left\|\mathbf{u}_{i,j}\right\|_2^2}\right)\\
		& \left[\tilde{\Omega}\left(\frac{\eta^2\sigma_n^2\left\|\mathbf{u}_{i,j}\right\|^2_2}{m}\right) + \tilde{\Omega}\left(\frac{\eta^2d\sigma_n^2\sigma_p^2}{m}\right)-\tilde{\mathcal{O}}\left(\frac{\eta}{m\sqrt{n}}\sqrt{d}\sigma_p^2 +\frac{\eta}{m}\sqrt{\frac{1}{n}}\left\|\mathbf{u}_{i,j}\right\|_2^2\right)\right]\\
  	\ge&\left(1-\mathcal{O}\left(\frac{\eta\gamma_{i,j}\left\|\mathbf{u}_{i,j}\right\|_2^2}{m}\right)\right)^T\mathcal{L}_{\mathcal{D}_{i,j}}(\mathbf{W}^{(0)}) \\
		& +\left[\tilde{\Omega}\left(\frac{\eta\sigma_n^2}{\gamma_{i,j}}\right) + \tilde{\Omega}\left(\frac{\eta d\sigma_n^2\sigma_p^2}{\gamma_{i,j}\left\|\mathbf{u}_{i,j}\right\|_2^2}\right)-\tilde{\mathcal{O}}\left(\frac{\sqrt{d}\sigma_p^2}{\sqrt{n}\gamma_{i,j}\left\|\mathbf{u}_{i,j}\right\|_2^2} +\frac{1}{\gamma_{i,j}}\sqrt{\frac{1}{n}}\right)\right]\\
		\ge& \left(1-\mathcal{O}\left(\frac{\eta\gamma_{i,j}\left\|\mathbf{u}_{i,j}\right\|_2^2}{m}\right)\right)^T\mathcal{L}_{\mathcal{D}_{i,j}}(\mathbf{W}^{(0)})+ \Omega\left(\frac{\eta d\sigma_n^2\sigma_p^2}{\gamma_{i,j}\left\|\mathbf{u}_{i,j}\right\|_2^2}\right) -\mathcal{O}\left(\frac{1}{\gamma_{i,j}}\sqrt{\frac{1}{n}}\right)
        %&+\mathcal{O}\left(\frac{\sqrt{d}\sigma_p^2}{\sqrt{n}\gamma_{i,j}\left\|\mathbf{u}_{i,j}\right\|_2^2}\right).
	\end{aligned}
\end{equation}
This completes the proof.
\end{proof}
\subsection{Proof of Theorem \ref{theorem: convergence_adv}}\label{appendix: theorem convergence_adv proof}
\begin{proof}
	Based on (\ref{equ: loss}), for any $(\mathbf{x},y) \sim \mathcal{D}_{i,j}$, we have
	\begin{equation}
		\begin{aligned}
			&\mathcal{L}\left(\mathbf{W}^{(t+1)}, \mathbf{x}+\boldsymbol{\zeta}^{(t+1)}\left(\mathbf{x}\right),y\right) - \mathcal{L}\left(\mathbf{W}^{(t+1)},\mathbf{x},y\right)\\
			=& \log\left(1+ \left(1-\text{prob}_y\left(\mathbf{W}^{(t+1)},\mathbf{x}\right)\right)\cdot\left(\exp\left(\tilde{\Delta}_{3-y}^{(t+1)}\left(\mathbf{x}\right)-\tilde{\Delta}_y^{(t+1)}\left(\mathbf{x}\right)\right)-1\right)\right),
		\end{aligned}
	\end{equation}
	where 
	\begin{align}
		\boldsymbol{\zeta}^{(t+1)}\left(\mathbf{x}\right) = \arg\max_{\left\|\boldsymbol{\zeta}\right\|_p\le \bar{\zeta}} \mathcal{L}\left(\mathbf{W}^{(t+1)},\mathbf{x}+\boldsymbol{\zeta},y\right),
	\end{align}
	and
	\begin{equation}
		\begin{aligned}
			&\tilde{\Delta}_{3-y}^{(t+1)}\left(\mathbf{x}\right) - \tilde{\Delta}_y^{(t+1)}	\left(\mathbf{x}\right)\\
			=&\underbrace{\frac{1}{m}\sum_{r=1}^{m} \sum_{j=1}^{2}\left[\sigma\left(
				\left\langle \mathbf{w}^{(t+1)}_{3-y,r},\mathbf{x}^{(j)}+\boldsymbol{\zeta}^{(t+1)}\left(\mathbf{x}\right)^{(j)} \right\rangle\right) - \sigma\left(\left\langle\mathbf{w}^{(t+1)}_{3-y,r},\mathbf{x}^{(j)} \right\rangle\right)\right]}_{A_5}\\
			&- \underbrace{\frac{1}{m}\sum_{r=1}^{m}\sum_{j=1}^{2} \left[\sigma\left(
				\left\langle \mathbf{w}^{(t+1)}_{y,r},\mathbf{x}^{(j)} + \boldsymbol{\zeta}^{(t+1)}\left(\mathbf{x}\right)^{(j)} \right\rangle\right) - \sigma\left(\left\langle\mathbf{w}^{(t+1)}_{y,r},\mathbf{x}^{(j)} \right\rangle\right)\right]}_{A_6}.
		\end{aligned}
	\end{equation}
 Then, we bound $A_5$ and $A_6$.

For the term $A_5$, with probability at least $1-\exp\left(-\Omega(d)\right)$, we have
	\begin{equation}
		\begin{aligned}
			A_5 
			\overset{(a)}{\le}& \frac{1}{m} \sum_{r=1}^{m}\sum_{j=1}^{2} \left|\left\langle\mathbf{w}_{3-y,r}^{(t+1)} ,\boldsymbol{\zeta}^{(t+1)}\left(\mathbf{x}\right)^{(j)}\right\rangle\right|\\
			\overset{(b)}{\le}& \frac{1}{m} \sum_{r=1}^{m}\sum_{j=1}^{2}  \left\|\mathbf{w}_{3-y,r}^{(t+1)}\right\|_2\left\|\boldsymbol{\zeta}^{(t+1)}\left(\mathbf{x}\right)^{(j)}\right\|_2\\
			\overset{(c)}{\le}& \frac{1}{m} \sum_{r=1}^{m} \sum_{j=1}^{2} \left\|\mathbf{w}_{3-y,r}^{(t+1)}\right\|_2\left\|\boldsymbol{\zeta}^{(t+1)}\left(\mathbf{x}\right)^{(j)}\right\|_p d^{1-\frac{1}{p}}\\
			\le& \frac{2}{m} \sum_{r=1}^{m}\left(\sum_{t'=0}^{t}  \left\|\mathbf{w}_{3-y,r}^{(t'+1)} - \mathbf{w}_{3-y,r}^{(t')}\right\|_2 + \left\|\mathbf{w}_{3-y,r}^{(0)}\right\|_2\right)\bar{\zeta} d^{1-\frac{1}{p}}\\
			\overset{(d)}{\le}&\mathcal{O}\left(\left[t\frac{\eta}{m} C+\frac{\eta}{m}\sqrt{td}\sigma_n+\sqrt{d}\sigma_0\right]\bar{\zeta} d^{1-\frac{1}{p}}\right),
		\end{aligned}
	\end{equation}
	where $(a)$ is because ReLU$(\cdot)$ is $1$-Lipschitz continuous; 
	$(b)$ is due to the Cauchy-Schwarz inequality;
	$(c)$ is because of the H\"older's inequality;
	$(d)$ is due to Lemma \ref{lemma: chi-squared}.
 
		Similarly, for the term $A_6$, with probability at least $1-\exp\left(-\Omega(d)\right)$, we have we have
	\begin{equation}
		\begin{aligned}
			A_6 \ge& -\frac{1}{m} \sum_{r=1}^{m}\sum_{j=1}^{2}  \left|\left\langle \mathbf{w}_{y,r}^{(t+1)} ,\boldsymbol{\zeta}^{(t+1)}\left(\mathbf{x}\right)^{(j)}\right\rangle\right|\\
			\ge& -\frac{1}{m} \sum_{r=1}^{m}\sum_{j=1}^{2} \left\|\mathbf{w}_{y,r}^{(t+1)} \right\|_2\left\|\boldsymbol{\zeta}^{(t+1)}\left(\mathbf{x}\right)^{(j)}\right\|_2\\
			\ge& -\frac{2}{m} \sum_{r=1}^{m}\left(\sum_{t'=0}^{t}  \left\|\mathbf{w}_{y,r}^{(t'+1)} - \mathbf{w}_{y,r}^{(t')}\right\|_2 + \left\|\mathbf{w}_{y,r}^{(0)}\right\|_2\right)\bar{\zeta} d^{1-\frac{1}{p}}\\
			\ge& -\Omega\left(\left[t\frac{\eta}{m} C+\frac{\eta}{m}\sqrt{td}\sigma_n+\sqrt{d}\sigma_0\right]\bar{\zeta} d^{1-\frac{1}{p}}\right).
		\end{aligned}
	\end{equation}
	
	Combing with the bounds of $A_5,A_6$, with probability at least $1-\exp(-\tilde{\Omega}(d))$, we have
	\begin{equation}
		\begin{aligned}
			&\tilde{\Delta}_{3-y}^{(t+1)}\left(\mathbf{x}\right) - \tilde{\Delta}_y^{(t+1)}	\left(\mathbf{x}\right)\le\mathcal{O}\left(\left[t\frac{\eta}{m}C+\frac{\eta}{m}\sqrt{td}\sigma_n+\sqrt{d}\sigma_0\right]\bar{\zeta} d^{1-\frac{1}{p}}\right).
		\end{aligned}
	\end{equation}
	Then, with probability at least $1-\exp(-\tilde{\Omega}(d))$, we have
	\begin{equation}
		\begin{aligned}
			&\mathcal{L}^\text{adv}_{\mathcal{D}_{i,j}}\left(\mathbf{W}^{(t+1)}\right) - \mathcal{L}_{\mathcal{D}_{i,j}}\left(\mathbf{W}^{(t+1)}\right)\\
			\le& \mathbb{E}_{\left(\mathbf{x},y\right)\sim \mathcal{D}_{i,j}}\left[\Gamma\left(\tilde{\Delta}_{3-y}^{(t+1)}\left(\mathbf{x}\right)-\tilde{\Delta}_y^{(t+1)}\left(\mathbf{x}\right)\right)\left(\tilde{\Delta}_{3-y}^{(t+1)}\left(\mathbf{x}\right)-\tilde{\Delta}_y^{(t+1)}\left(\mathbf{x}\right)\right)\right]\\
			\le&\mathcal{O}\left(\left[t\frac{\eta}{m}C+\frac{\eta}{m}\sqrt{td}\sigma_n+\sqrt{d}\sigma_0\right]\bar{\zeta} d^{1-\frac{1}{p}}\right).
		\end{aligned}
	\end{equation}
	Combining with Theorem \ref{theorem: convergence} and setting parameters with Condition \ref{condition}, with probability at least $1-\exp(-\tilde{\Omega}(d))$, we have
	\begin{equation}
		\begin{aligned}
			\mathcal{L}^\text{adv}_{\mathcal{D}_{i,j}}\left(\mathbf{W}^{(T)}\right)
			\le& \bar{L}_{i,j} + \mathcal{O}\left(\left[\frac{T}{m} C+\frac{\sqrt{Td}}{m}\sigma_n+\sqrt{d}\sigma_0\right]\bar{\zeta} d^{1-\frac{1}{p}}\right).
		\end{aligned}
	\end{equation}
	This completes the proof.
\end{proof}

\section{Proof of Proposition \ref{proposition: pt ft}}
\begin{proof}
	We have
	\begin{align}
		&\sigma(\langle\tilde{\mathbf{w}}_{1,r},\mathbf{u}'_{1}\rangle) = C_1 \cos\theta \left\|\mathbf{u}_{1}'\right\|_2\left\|\mathbf{u}_{1}\right\|_2,\\
		&0 \le\sigma(\langle\tilde{\mathbf{w}}_{1,r},\boldsymbol{\xi}\rangle) \le C_3 \sigma_p^2,\\
		& \sigma(\langle\tilde{\mathbf{w}}_{2,r},\mathbf{u}'_{1}\rangle) = C_1\sin\theta \left\|\mathbf{u}'_{1}\right\|_2\left\|\mathbf{u}_{2}\right\|_2,\\
		&0 \le\sigma(\langle\tilde{\mathbf{w}}_{2,r},\boldsymbol{\xi}\rangle) \le C_3 \sigma_p^2,\\
		&\sigma(\langle\tilde{\mathbf{w}}_{1,r},\mathbf{u}'_{2}\rangle) = 0,\\
		& \sigma(\langle\tilde{\mathbf{w}}_{2,r},\mathbf{u}'_{2}\rangle) = C_1\cos\theta \left\|\mathbf{u}_{2}\right\|_2\left\|\mathbf{u}'_{2}\right\|_2.
	\end{align}
	
	Using the above inequalities (equalities), we have
	\begin{equation}
		\begin{aligned}
			\mathcal{L}_{\mathcal{D}_2}(\tilde{\mathbf{W}}) \le& -\frac{1}{2}\ln\left(\frac{\exp(C_1 \cos\theta \left\|\mathbf{u}'_{1}\right\|_2\left\|\mathbf{u}_{1}\right\|_2)}{\exp(C_1 \cos\theta \left\|\mathbf{u}'_{1}\right\|_2\left\|\mathbf{u}_{1}\right\|_2)+\exp(C_1\sin\theta \left\|\mathbf{u}'_{1}\right\|_2\left\|\mathbf{u}_{2}\right\|_2+ C_3\sigma_p^2)}\right)\\
			&-\frac{1}{2}\ln\left(\frac{\exp(C_1 \cos\theta \left\|\mathbf{u}'_{2}\right\|_2\left\|\mathbf{u}_{2}\right\|_2)}{\exp(C_1 \cos\theta \left\|\mathbf{u}'_{2}\right\|_2\left\|\mathbf{u}_{2}\right\|_2)+\exp( C_3\sigma_p^2)}\right).
		\end{aligned}
	\end{equation}
	Based on Theorem \ref{theorem: convergence} and $\left\|\mathbf{u}_1\right\|_2=\left\|\mathbf{u}_1\right\|_2=\left\|\mathbf{u}_1'\right\|_2 =\left\|\mathbf{u}_2'\right\|_2$, we have
	\begin{equation}
		\begin{aligned}
			&\mathcal{L}_{\mathcal{D}_\textnormal{ft}}(\mathbf{W}^{(T)})\le 
			\exp\!\left(\!-\Omega\left(\!\frac{\Lambda_i\left\|\mathbf{u}_{i}\right\|_2^2}{m}T\right)\!\right)\!\cdot\tilde{L}+\mathcal{O}\left(\frac{\sqrt{d}}{\sqrt{n}\Lambda_{i}}\right)+\mathcal{O}\left(\frac{m\sqrt{d}\sigma_n}{\Lambda_i\left\|\mathbf{u}_i\right\|_2}\right),
		\end{aligned}
	\end{equation}
	This completes the proof. 
\end{proof}

\end{document}